\documentclass{article}

\usepackage[english]{babel}

\usepackage[letterpaper,top=2cm,bottom=2cm,left=2.5cm,right=2.5cm,marginparwidth=1.75cm]{geometry}

\usepackage{setspace}
\usepackage{amsmath}
\usepackage{graphicx}
\usepackage[colorlinks=true, allcolors=blue]{hyperref}
\usepackage{multirow}

\usepackage{caption}
\usepackage{subcaption}
\usepackage{dsfont}

\usepackage{algorithm}
\usepackage{algpseudocode}
\usepackage{authblk}
\usepackage{enumerate}
\usepackage{amsthm}

\usepackage{cancel}
\usepackage{amsmath}
\usepackage{amsfonts}
\usepackage{amssymb}
\usepackage{graphicx}
\usepackage{enumitem}
\usepackage{mathtools} %
\usepackage{ulem}
\usepackage{natbib}
\usepackage{amsmath}
\mathtoolsset{showonlyrefs} %

\DeclareMathOperator*{\argmax}{arg\,max}

\newtheorem{theorem}{Theorem}
\newtheorem{assumption}{Assumption}
\newtheorem{proposition}{Proposition}

\newtheorem{lemma}{Lemma}

\newtheorem{definition}{Definition}

\usepackage{hyperref}
\hypersetup{
    colorlinks=true,
    linkcolor=blue,
    filecolor=magenta,      
    urlcolor=cyan,
    pdftitle={Overleaf Example},
    pdfpagemode=FullScreen,
}

\title{On the Optimality of Tracking Fisher Information \\in Adaptive Testing with Stochastic Binary Responses}

\author[1]{Sanghwa Kim}
\author[2]{Dohyun Ahn}
\author[3]{Seungki Min}

\affil[1]{\small Kim Jaechul Graduate School of AI, KAIST, Seoul, Republic of Korea}
\affil[2]{Department of Systems Engineering and Engineering Management, The Chinese University of Hong Kong, Shatin, N.T., Hong Kong}
\affil[3]{Business School, Seoul National University, Seoul, Republic of Korea}
\date{}

\begin{document}
\maketitle

\begin{abstract}
We study the problem of estimating a continuous ability parameter from sequential binary responses by actively asking questions with varying difficulties, a setting that arises naturally in adaptive testing and online preference learning.
Our goal is to certify that the estimate lies within a desired margin of error, using as few queries as possible.
We propose a simple algorithm that adaptively selects questions to maximize Fisher information and updates the estimate using a method-of-moments approach, paired with a novel test statistic to decide when the estimate is accurate enough.
We prove that this Fisher-tracking strategy achieves optimal performance in both fixed-confidence and fixed-budget regimes, which are commonly invested in the best-arm identification literature.
Our analysis overcomes a key technical challenge in the fixed-budget setting---handling the dependence between the evolving estimate and the query distribution---by exploiting a structural symmetry in the model and combining large deviation tools with Ville's inequality.
Our results provide rigorous theoretical support for simple and efficient adaptive testing procedures.
\end{abstract}

\section{Introduction}\label{section:intro}

Adaptive testing and sequential estimation problems have recently gained substantial attention due to their foundational role in modern artificial intelligence and interactive systems.
Prominent applications include online preference learning, where systems dynamically adapt to user feedback to refine personalized recommendations, and reinforcement learning from human feedback (RLHF), which aims to align AI agents with human values by adaptively querying users.
In these contexts, the main focus is to efficiently extract maximal information from human responses, which are inherently stochastic and limited in quantity.

Among various types of such problems, this work particularly considers a fundamental yet illustrative case involving stochastic binary responses.
Here, a decision-maker sequentially selects questions of varying difficulty from a continuous pool to pose to a candidate and aims to efficiently estimate the candidate's ability (represented by an unknown continuous parameter) by utilizing the binary feedback (e.g., correct/incorrect) collected, which depends probabilistically on the candidate's ability and the question's difficulty.
This setup is arguably the simplest scenario that captures the essence of continuous parameter estimation under uncertainty, making it an ideal benchmark for developing fundamental theoretical insights and practical algorithms.

Variants of this fundamental adaptive estimation problem have been studied in several communities.
In psychometrics, the literature of adaptive testing---particularly, computerized adaptive testing (CAT)---has traditionally relied on algorithms that select questions based on Fisher information to efficiently assess candidate ability levels based on item response theory.
Meanwhile, in the realm of online learning, related problems have been extensively investigated under the frameworks of best-arm identification and thresholding bandits, which focus on developing algorithms that achieve optimal sample complexity in discrete decision spaces.
Specifically, seminal contributions within this area have established information-theoretic lower bounds and corresponding optimal adaptive sampling algorithms under fixed-confidence settings.
Additionally, the broader operations research literature on simulation optimization and ordinal optimization has contributed techniques that exploit large deviation principles (LDPs) to construct deterministic approximations of probabilistic objectives for optimal sample allocation in fixed-budget settings.

In this paper, we propose and rigorously analyze an adaptive testing algorithm that integrates a method-of-moments estimator for continuously updating the candidate's ability estimate and selects questions that track Fisher-optimal informativeness.
We couple this adaptive sampling with a tailored stopping rule based on a novel test statistic.
Our theoretical analysis employs the theory of large deviations and information-theoretic arguments to precisely characterize the asymptotic performance of the algorithm, explicitly quantifying the decay rate of the failure probability and the growth rate of expected sample complexity.

In particular, we make the following key contributions:
\begin{enumerate}
    \item \emph{Asymptotic optimality of a simple, intuitive algorithm}:
    We show that a conceptually simple adaptive algorithm---selecting queries that maximize Fisher information---achieves asymptotic optimality in both \textit{fixed-budget} and \textit{fixed-confidence} settings.
    This Fisher-optimal query strategy, long employed as a heuristic in psychometric testing and active learning, had previously lacked theoretical justification in the context of continuous adaptive estimation.
    Our analysis establishes that this query rule, when paired with a method-of-moments estimator, achieves information-theoretic lower bounds on error probability decay and sample complexity.
    To our knowledge, this is the first result that theoretically demonstrates the optimality of such a strategy in continuous query settings.

    \item \emph{Design of a novel test statistic for continuous parameter estimation}:
    We introduce a novel test statistic tailored to certify that a continuous parameter lies within a fixed margin of accuracy.
    Our statistic incorporates large deviations bounds into the framework of the generalized likelihood-ratio test to evaluate predictive discrepancies between the current estimate and alternative hypotheses while compensating for statistical uncertainty using an explicit deviation correction function derived from the large deviations bounds.
    The proposed statistic consequently yields a stopping rule that is both theoretically grounded and computationally tractable for continuous parameter estimation.
    Notably, this construction contrasts with the test statistics used in best-arm identification, which inherently assume and leverage a discrete action space, and direct extensions of these statistics to a continuous space prove unreliable.
    
    \item \emph{Countering estimate-query endogeneity}:
    A central technical challenge in adaptive estimation under fixed-budget settings lies in the estimate-query endogeneity---a phenomenon where, at each stage, the query distribution depends on the parameter estimate, and the empirical concentration of the estimate depends on data gathered via the query.
    As noted by \cite{wang2024best}, ``Due to the intricate dependence between the sampling and the empirical reward processes, deriving such an LDP is very challenging.''
    To circumvent this issue, we (1) leverage Ville's inequality within our large deviations analysis to bound the time-uniform deviation of the estimator with high probability and (2) design a continuous cyclical relationship between the estimate and query and establish a similar time-uniform deviation bound for the query.
    We expect this methodological advance to stimulate further developments in large deviations analysis across broader problem instances.
\end{enumerate}

\subsection{Literature review}

\paragraph{Computerized adaptive testing (CAT).}
As alluded to earlier, our problem can be considered as a version of the CAT problem, in which query selection is dynamically adjusted based on the test history to estimate their ability accurately and efficiently.
There have been several notable developments in this literature.
Specifically, \cite{bartroff2008modern} explore a testing framework that aims to categorize examinees as either `masters' or `non-masters' based on whether their ability exceeds a predetermined threshold, rather than estimating their ability as in our problem.
\cite{bassamboo2023learning} generalize this framework for classifying candidates' ability levels into predetermined intervals, design a related stopping rule based on hypothesis testing and generalized likelihood ratio tests, and prove its asymptotic optimality in terms of sample complexity.
While their methodology could potentially be applied to our problem by specifying sufficiently small interval lengths, several key differences remain in algorithm design, performance guarantees, and model coverage; see Appendix~\ref{appendix:bassamboo} for a more detailed comparison between their work and ours.

Under two-parameter and three-parameter logistic models, \cite{chang2009nonlinear} develop an adaptive testing algorithm that tracks queries maximizing Fisher information and provide the consistency and asymptotic normality of the associated maximum-likelihood estimator.
\cite{yang2024note} propose to use a widely known stochastic approximation method, called the Robbins-Monro algorithm, and numerically demonstrate its superiority over traditional maximum information methods under Rasch and two-parameter logistic models when it comes to reducing mean squared error.
By contrast, rather than being limited to specific response models, our algorithm and analysis apply to a fairly general class of response models that encompasses widely used examples such as logistic and algebraic models. More importantly, we provide theoretical performance guarantees of the proposed algorithm in this general setup.

\paragraph{Active learning.}
This work is also in line with active learning frameworks, where learners strategically query an oracle to maximize the amount of information obtained from queries within a fixed budget or to minimize the effort to achieve a target level of information~\citep{settles2009active, houlsby2011bayesian}.
The active learning process typically begins by carefully selecting queries for labeling. Upon receiving the query results, the newly labeled data point is added to the labeled set, and the algorithm learns from this new information and uses it to inform subsequent query selection.
This is useful when labeled data is scarce, difficult to obtain, or expensive to annotate.
Our study particularly shares a common thread with \cite{sourati2017asymptotic} as it provides asymptotic analyses of Fisher-information-based active learning methods. However,  their work, consistent with the classical active learning literature, focuses on parametric classification problems, which differ from our estimation problem. In addition, Fisher information is used differently in query rule design: their rule seeks to minimize the ratio of Fisher information for the true and proposal distributions, whereas ours aims to maximize the Fisher information corresponding to our response model. Lastly, their paper does not analyze performance guarantees in terms of the convergence rate and sample complexity associated with the error probability.

\paragraph{Stochastic bisection search.}
Our setting, which focuses on identifying an unknown parameter on the real line from noisy binary responses, is closely related to the stochastic bisection search problem. In this context, query responses are inherently noisy, meaning that observed outcomes may deviate from the true underlying state. However, the standard literature on stochastic bisection search presents two key differences from our work in terms of modeling and handling this noise. In particular, \cite{frazier2019probabilistic} and \cite{waeber2013bisection} model the noise such that querying a point $X_t$ is subject to a fixed probability of error, and \cite{jedynak2012twenty} consider a situation where querying a set results in an incorrect response with a probability contingent on the measure of the set. To address this randomness, these works commonly introduce and update a Bayesian belief over the search interval after each query. Their algorithms' optimality is then established with respect to the entropy of the resulting posterior distribution.

\paragraph{Best-arm identification.}

Despite the difference in objectives, our analysis utilizes technical tools akin to those used for fixed-budget and fixed-confidence best-arm identification problems---also known as ordinal optimization or ranking-and-selection---in the online learning literature. In contrast to our goal, these problems aim to sequentially draw samples from a finite collection of stochastic alternatives to identify the best alternative, where the definition of the ``best'' depends on problem instances and contexts.

In the fixed-budget setting where decision-makers seek to minimize the error probability by judiciously allocating a given sampling budget, \cite{glynn2004large} derive an LDP result for the performance of static sample allocation schemes in best-arm identification.
Inspired by this pioneering work, there have been many follow-up studies to address tractability issues~\citep{shin2018tractable, chen2022BOLD}, feasibility determination~\citep{szechtman2008new}, performance improvement~\citep{zhou:23SIndex}, and model misspecification~\citep{ahn2025feature}.
However, none of these works characterize the asymptotic decay rate of the error probability for dynamic sampling policies, which is widely recognized as challenging due to the interdependence between samples and allocation decisions, analogous to the estimate-query endogeneity in our problem.
Very recently, \cite{wang2024best} establish a connection between the LDP based on historical samples and that derived from empirical allocation in adaptive policies.
Nonetheless, the resulting upper bound of the error probability is not tight enough to analyze the asymptotic performance of these policies.
While limited to our adaptive testing setup, we address this issue by characterizing a tight upper bound of the error probability that asymptotically achieves the information-theoretic lower bound.

The fixed-confidence setting focuses on identifying the best alternative with a specified confidence level while taking as few samples as possible.
In this context, well-designed stopping rules are essential for achieving optimal sample complexity.
In the existing literature, most analyses on the performance of stopping rules either leverage the discrete and finite nature of decision spaces \citep[see, e.g.,][]{garivier2016optimal, kaufmann2016complexity} or exploit specific structural modeling assumptions, such as linear response models \citep{jedra2020optimal}.
In contrast, as previously highlighted, we demonstrate the asymptotic optimality of our stopping rule designed for a continuous decision space and a broad class of response models.

\subsection{Paper organization}

In Section~\ref{section:problem}, we formally introduce the problem setup, including the adaptive testing framework and relevant modeling assumptions.
Section~\ref{section:algorithm} presents our proposed algorithm, referred to as \texttt{FIT-Q}, with a detailed description of its query, stopping, and prediction rules in both fixed-budget and fixed-confidence settings.
Section~\ref{section:analysis} analyzes theoretical performance guarantees of the \texttt{FIT-Q} algorithm; in particular, we demonstrate the algorithm's asymptotic optimality.
In Section~\ref{section:experiment}, we validate our theoretical findings through numerical experiments.
Finally, Section~\ref{section:conclusion} concludes the paper with a summary of our findings and a discussion on potential directions for future research.

\subsection{Notation}

We denote by $\mathbb{P}_{\theta_*}^{\pi}$ the probability measure induced by an adaptive testing algorithm $\pi$ operating under the data-generating distribution parameterized by $\theta_*$.
Correspondingly, $\mathbb{E}_{\theta_*}^{\pi}$ represents the expectation with respect to the probability measure $\mathbb{P}_{\theta_*}^{\pi}$.
The projection of a point $x$ onto a set $\mathcal{X}$ is defined as $\text{proj}_{\mathcal{X}} (x)$.
The Kullback–Leibler (KL) divergence between two Bernoulli distributions with parameters $p$ and $q$ is denoted by $d(p \mid q)\coloneqq p \log ({p}/{q}) + (1-p) \log \big(({1-p})/({1-q})\big)$.
We express $g(\epsilon) = o \big(f(\epsilon) \big)$ when ${g(\epsilon)}/{f(\epsilon)}\to 0$ as $\epsilon \to 0$, and we write $g(\epsilon) = O(f(\epsilon))$ when there exist $C>0$ and $\epsilon_0>0$ such that for all $\epsilon \in (0, \epsilon_0)$, $|g(\epsilon)| \leq C | f(\epsilon)|$.

\section{Problem Setup}\label{section:problem}

We consider a decision-maker who aims to identify a candidate's ability by asking a series of adaptively chosen questions.
The candidate's ability, denoted by $\theta_* \in \mathbb{R}$, is a real number that is unknown to the decision-maker.
In each time $t = 1, 2, \cdots$, the decision-maker asks the candidate a question parametrized by $X_t \in \mathcal{X}$, where $X_t$ represents the difficulty of the question at time $t$, and $\mathcal{X}\subset\mathbb{R}$ denotes a fixed set of all possible difficulty levels, assumed to be a closed interval $[x_L, x_R]$ for some real numbers $x_L < x_R$.
The decision-maker then immediately observes the candidate's response $Y_t \in \{0,1\}$, where $Y_t=1$ indicates a correct answer and $Y_t=0$ an incorrect one.
We assume that the decision-maker postulates a static probabilistic model governing the candidate's responses, which is leveraged to obtain the estimate $\hat{\theta}_t$ of the candidate's ability $\theta_*$.
More details on the model are described below.

\paragraph{Random response model.}
Considering the stochastic nature of candidate responses and drawing on principles from item response theory \citep{bock2021item}, we assume that the probability of a correct response depends on the gap between the candidate's ability and the difficulty of the question.
Specifically, the candidate's responses, $\{Y_t\}_{t=1}^\infty$, are mutually independent conditional on the questions, $\{X_t\}_{t=1}^\infty$, and the conditional probability of a correct response at each time $t$ is given by
\begin{equation}\label{eq:response model}
    \mathbb{P}_{\theta_*}\big( Y_t = 1 | X_t = x \big) = f( \theta_* - x ),    
\end{equation}
for some known function $f:\mathbb{R} \rightarrow [0,1]$, which we call the \emph{response function}.
In the following assumption, we impose mild regularity conditions on the response functions.

\begin{assumption}\label{assumption:model}
    The response function $f: \mathbb{R} \rightarrow [0,1]$  satisfies the following conditions:
    \begin{enumerate}[label=(\alph*)]
        \item\label{cond:increasing} $f$ is strictly increasing, i.e. $f'(z) >0$ for all $z\in\mathbb{R}$; and
        \item\label{cond:C2} $f$ is thrice continuously differentiable on $\mathbb{R}$;
        \item\label{cond:h} $\lim_{z \to \infty} h(z) = \lim_{z \to - \infty} h(z) = 0$, where $h(z) \coloneqq {f'(z)^2}/\{f(z) \big(1 - f(z) \big)\}$.
    \end{enumerate}
\end{assumption}
Condition~\ref{cond:C2} ensures the smoothness of the response function $f$. Condition~\ref{cond:increasing} is straightforward because a candidate is more likely to provide a correct response when the candidate possesses a higher ability or when the question is easier.
Condition~\ref{cond:h} ensures the existence of an optimal query corresponding to the candidate, which is defined in Section \ref{section:algorithm}.
Indeed, condition~\ref{cond:h} is generally satisfied for any analytically defined function that meets conditions~\ref{cond:C2} and~\ref{cond:increasing}:
see, e.g., the following canonical cases
\begin{itemize}
    \item Logistic model: $f(z) = 1/(1+e^{-z})$;
    \item Algebraic model: $f(z) = 0.5/(1+|z|^{k})^{1/k} + 0.5$.
\end{itemize}
Nonetheless, we include condition~\ref{cond:h} for completeness, acknowledging that violations of this condition may arise in extreme and practically irrelevant scenarios.

\paragraph{Decision process.}
The decision-maker makes sequential decisions on which question to ask and when to stop in an adaptive manner.
To formally describe this adaptiveness, let $H_t \coloneqq \left\{ (X_s, Y_s) \right\}_{s=1}^{t}$ indicate the interaction history realized up to time $t$, and let $\mathbb{F} \coloneqq \left\{ \mathcal{F}_t \coloneqq \sigma(H_t) \right\}_{t \in \mathbb{N}}$ represent the filtration generated by this history.
Then, the platform's decision-making algorithm can be characterized by the following three components:
\begin{enumerate}
    \item \textbf{Query rule} determines a question $X_t \in \mathcal{X}$ at time $t$ based on previously revealed information $H_{t-1}$. This enables the stochastic process $\{ X_t \}_{t \in \mathbb{N}}$ to be adapted with respect to the filtration $\mathbb{F}$.
    
    \item \textbf{Stopping rule} determines whether to stop or continue asking questions.
    Let $\tau$ denote the number of questions asked until stopping, then $\tau$ is a stopping time adapted to the filtration $\mathbb{F}$.
    In the fixed-budget setting where the decision-maker is allowed to ask exactly $T$ questions, no strategic stopping rule is needed, and thus, we simply have $\tau = T$.
    In the fixed-confidence setting, the decision-maker may stop at any time upon determining that sufficient responses have been collected. In this case, we require the test process to terminate almost surely, i.e., $\tau < \infty$ with probability 1. 
    
    \item \textbf{Prediction rule} defines the final estimate $\hat{\theta}_\tau$ of the ability $\theta_*$ at the stopping time $\tau$ such that the random variable $\hat{\theta}_\tau$ is measurable with respect to $\mathcal{F}_\tau$.
\end{enumerate}

We denote such a decision-making algorithm by $\pi$ and use the notations $\mathbb{P}_{\theta_*}^\pi( \cdot )$ and $\mathbb{E}_{\theta_*}^\pi[\cdot]$ to indicate the probability measure and the expectation operator, respectively, associated with the candidate's level $\theta_*$ and the decision-making algorithm $\pi$. In what follows, we provide a formal description of fixed-budget and fixed-confidence settings introduced in Section~\ref{section:intro}.

\paragraph{Fixed-budget (FB) setting.}
In the FB setting, the decision-maker  seeks to minimize the probability of making erroneous prediction after asking a fixed number $T \in \mathbb{N}$ of questions based on an algorithm $\pi$, where $T$ is referred to as the budget.
We call this probability the \emph{failure probability} of the algorithm $\pi$, and it is formally defined as
\begin{equation}
	p^\text{FB}(\pi_{\epsilon, T}; \theta_*, \epsilon, T) \coloneqq \mathbb{P}_{\theta_*}^{\pi_{\epsilon, T}} \big( | \hat{\theta}_T - \theta_* | > \epsilon \big),
\end{equation}
where $\epsilon \in \mathbb{R}$ is an error margin allowed for prediction.
Since this failure probability does not admit a closed-form expression, it is challenging to construct and analyze an algorithm that minimizes $p^\text{FB}(\pi_{\epsilon, T}; \theta_*, \epsilon, T)$.
Instead, as a well-established alternative in the literature, this paper focuses on characterizing and maximizing the asymptotic decay rate of the failure probability, given by
\begin{equation}
    -\lim_{T \to \infty} \frac{1}{T} \log p^\text{FB}(\pi_{\epsilon,T}; \theta_*, \epsilon, T),
\end{equation}
which we will utilize as a performance metric to analyze the optimality of our suggested algorithm.
Here, we use the notation $\pi_{\epsilon, T}$ to clarify that the algorithm may depend on the error margin $\epsilon$ and the budget $T$.
The subscripts $\epsilon$ and $T$ in $\pi_{\epsilon, T}$ may be omitted if there is no confusion.

\paragraph{Fixed-confidence (FC) setting.}
The decision-maker in the FC setting aims to use the fewest possible questions while guaranteeing that the failure probability does not exceed a target confidence level $\delta \in (0,1)$.
More specifically, we redefine the failure probability for the FC setting as
\begin{equation}
    p^\text{FC}(\pi_{\epsilon, \delta}; \theta_*, \epsilon) \coloneqq \mathbb{P}_{\theta_*}^{\pi_{\epsilon, \delta}} \big( | \hat{\theta}_\tau - \theta_* | > \epsilon \big),
\end{equation}
where $\tau$ is the stopping time of the algorithm $\pi$.
The decision-maker's goal is to minimize the expected stopping time, $\mathbb{E}_{\theta_*}^{\pi_{\epsilon,\delta}}[\tau]$, subject to the constraint $p^\text{FC}(\pi_{\epsilon,\delta}; \theta_*, \epsilon) \leq \delta$.
In our analysis, we investigate the following growth rate of the expected stopping time in an asymptotic regime where the target error rate $\delta$ approaches to zero:
\begin{equation}
    \lim_{\delta \rightarrow 0} \frac{\mathbb{E}_{\theta_*}^{\pi_{\epsilon,\delta}} [\tau]}{\log(1/\delta)}.
\end{equation}
In this setting, the algorithm may depend on the error margin $\epsilon$ and the confidence level $\delta$, which is reflected in the notation $\pi_{\epsilon, \delta}$.
When the context is clear, the subscripts $\epsilon$ and $\delta$ are omitted for brevity.

\section{Proposed Algorithm}\label{section:algorithm}

In this section, we introduce two versions of our proposed algorithm (\texttt{FIT-Q}) for FB and FC settings, respectively.
The naming of the algorithm stems from the fact that both versions commonly implement the \textbf{Fisher information tracking query rule} built upon the method of moments estimator (MME).
On top of that, a tailored stopping rule is designed and incorporated specifically for the FC setting.

\subsection{Fisher Information}

Our query rule utilizes Fisher information as a measure of the informativeness of a question.
Let $I(x ; \theta_*)$ be the Fisher information quantifying the amount of information that a single random response to a question $x \in \mathcal{X}$ carries about the candidate's ability $\theta_*$.
Formally, it is defined as the variance of the score:
\begin{equation}
	I(x ; \theta_*) 
        \coloneqq \mathbb{E}_{\theta_*} \left[ \left. \left( \left. \frac{\partial}{\partial \theta} \log p_\theta(Y|x) \right|_{\theta = \theta_*} \right)^2 \right| X=x \right],
\end{equation}
where $p_\theta( \cdot |x)$ is the probability mass function of the response $Y$ when question $x$ is given to the candidate with ability level $\theta$.
In our assumed response model in~\eqref{eq:response model}, we have $p_\theta(y|x) = f(\theta - x)^y \big(1 - f(\theta-x) \big)^{1-y}$ for $y \in \{0,1\}$, and therefore,
\begin{equation} \label{eq:fisher}
    I(x; \theta_*) = \frac{f'(\theta_* - x) ^2}{ f(\theta_* - x) \big( 1 - f(\theta_* - x) \big)}.
\end{equation}

We define the \textbf{hindsight Fisher-optimal query} $x_* \in \mathbb{R}$ as the unconstrained maximizer of Fisher information $I(x; \theta_*)$:
\begin{equation} \label{eq:fisher-query}
    x_* \coloneqq \theta_* - z_*, \quad \text{where }
    z_* \in \argmax_{z \in \mathbb{R}} \frac{f'(z)^2}{ f(z) \big( 1 - f(z) \big)}.
\end{equation}
The query $x_*$ is the optimal question that the decision maker would have kept asking if the candidate's ability $\theta_*$ were known; see Section \ref{subsec:static-query} for a detailed discussion on the significance of this query.
Here, $z_*$ is an auxiliary variable determined solely by the response function $f(\cdot)$, independent of $\theta_*$.
Its existence is guaranteed under Assumption \ref{assumption:model}.
For example, if $f$ is a logistic function, we have $z_*=0$, $x_* = \theta_*$, and $\mathbb{P}(Y_t=1 | X_t = x_*) = f(x_* - \theta_*) = f(z_*) = 0.5$, in which case the hindsight Fisher-optimal query $x_*$ corresponds to the question for which the student answers correctly with a 50\% chance.

\subsection{\texttt{FIT-Q} Algorithm}

We begin by presenting the pseudocode for the two versions of our \texttt{FIT-Q} algorithm for the FB and FC settings:

\begin{algorithm}[ht]
	\caption{\texttt{FIT-Q} Algorithm for Fixed-Budget Setting}\label{algo:FB}
	\begin{algorithmic}[1]
		\State \textbf{Input:} time budget $T$, tolerance $\epsilon$.
            \State Calculate $z_*$ using \eqref{eq:fisher-query}.
		\For{$t=1,\cdots,T$}
			\State Calculate the method of moments estimate $\hat{\theta}_{t-1}$ by solving \eqref{eq:estimate}.
			\State Asks a question $X_t = \text{proj}_\mathcal{X}  (\hat{\theta}_{t-1} - z_*)$ and receives feedback $Y_t$.
		\EndFor
		\State Report $\hat{\theta}_T$.
	\end{algorithmic}
\end{algorithm}

\begin{algorithm}[ht]
	\caption{\texttt{FIT-Q} Algorithm for Fixed-Confidence Setting}\label{algo:FC}
	\begin{algorithmic}[1]
		\State \textbf{Input:} confidence level $\delta$, tolerance $\epsilon$.
            \State Calculate $z_*$ using \eqref{eq:fisher-query}
		\For{$t=1,2,\ldots$}
			\State Calculate the method of moments estimate $\hat{\theta}_{t-1}$ by solving \eqref{eq:estimate}.
			\State Asks a question $X_t = \text{proj}_\mathcal{X}  (\hat{\theta}_{t-1} - z_*)$ and receives feedback $Y_t$.
                \State Calculate the test statistics $Z_t^\epsilon$ using \eqref{eq:test-stats}.
			\If{$Z_t^\epsilon > \log(2/\delta)$}
                    \State Set $\tau \gets t$
				\State \textbf{break}
			\EndIf
		\EndFor
		\State Report $\hat{\theta}_{\tau}$.
	\end{algorithmic}
\end{algorithm}

\paragraph{Method of moments estimator (MME).}
The \texttt{FIT-Q} algorithm adopts the method of moments estimator (MME) to estimate the candidate's ability $\theta_*$.
The MME is the value that makes the predicted mean of the outcomes match their empirical mean.
More formally, given a sequence of realized question-response pairs $\{ (X_s, Y_s)\}_{s=1}^t$, the MME, denoted by $\hat{\theta}_t$, is defined as
\begin{equation}\label{eq:estimate}
	\hat{\theta}_t \coloneqq \inf \left\{ \theta \in \mathbb{R} : \sum_{s=1}^{t} f(\theta - X_s) \geq \sum_{s=1}^{t} Y_s \right\},
\end{equation}
which takes values on an extended real number line.
Under Assumption \ref{assumption:model}, we have $\hat{\theta}_t = -\infty$ if and only if $ \inf_{z \in \mathbb{R}} f(z) \geq \frac{1}{t}\sum_{s=1}^{t} Y_s$ (e.g., the candidate has never answered correctly), and $\hat{\theta}_t = \infty$ if and only if $\sup_{z \in \mathbb{R}} f(z) \leq \frac{1}{t} \sum_{s=1}^{t} Y_s$ (e.g., the candidate has answered all questions correctly).
Excluding these trivial cases, the estimate $\hat{\theta}_t$ is given by the unique solution satisfying
\begin{equation} \label{eq:mme}
    \sum_{s=1}^{t} f(\hat{\theta}_t - X_s) = \sum_{s=1}^{t} Y_s,
\end{equation}
due to the monotonicity of the response function $f$.
The MME can be computed efficiently via the bisection method or Newton's method.

\paragraph{Fisher information tracking query rule.}
We consider a query rule that adaptively tracks the hindsight Fisher-optimal query $x_* = \theta_* - z_*$, by simply replacing the unknown true ability $\theta_*$ with its estimated value $\hat{\theta}_t$.
That is, at each time step $t$, it asks a query $\hat{\theta}_{t-1} - z_*$ if feasible.
In cases where the query is not feasible, the query rule performs a projection to select the feasible question closest to the value $\hat{\theta}_{t-1} - z_*$:\footnote{
    One may consider solving the constrained optimization problem in every single time period, i.e., $X_t = \argmax_{x \in \mathcal{X}} I(x; \hat{\theta}_{t-1})$.
    Nonetheless, we believe that its additional benefit, when compared to the suggested projection scheme in \eqref{eq:projection}, is marginal, whereas its computational cost is significantly higher.
}
\begin{equation}\label{eq:projection}
	X_t = \text{proj}_{\mathcal{X}}(\hat{\theta}_{t-1} - z_*)
        =
	\begin{cases}
		\hat{\theta}_{t-1} - z_* & \text{if} \;\; \hat{\theta}_{t-1} - z_* \in [x_L, x_R], \\
		x_L & \text{if} \;\; \hat{\theta}_{t-1} - z_* < x_L, \\
		x_R & \text{if} \;\; \hat{\theta}_{t-1} - z_* > x_R.
	\end{cases}
\end{equation}

Recall that the hindsight Fisher-information optimal query $x_*$ is the question that maximizes the Fisher information, $I(x; \theta_*)$.
Accordingly, the chosen query $X_t$ is designed to maximize the predicted Fisher information, $I(x; \hat{\theta}_{t-1})$.
Hence, if the estimate is accurate, the algorithm will select the question that yields the most information.
While the concept of tracking Fisher-optimal queries has been continuously adopted across various problems in the literature~\citep{bartroff2008modern,chang2009nonlinear, eggen1999item}, our emphasis is on rigorously justifying this query rule in our setup by providing theoretical optimality guarantees, which we shall see in Section~\ref{section:analysis}.

\paragraph{Stopping rule} (FC setting only).
In the fixed-confidence setting, a carefully designed stopping rule is required in order for the algorithm to determine whether learning has progressed enough.
Our suggested stopping rule relies on a test statistic $Z_t^\epsilon$, defined as
\begin{equation} \label{eq:test-stats}
    Z_t^{\epsilon} \coloneqq \min_{\theta \in \{\hat{\theta}_t - \epsilon, \hat{\theta}_t + \epsilon \}} \sum_{s=1}^t \left[ \lambda_* \left| f(\hat{\theta}_t - X_s) - f(\theta - X_s) \right| - \phi( \lambda_*, f(\theta - X_s) ) \right],
\end{equation}
where
\begin{equation}\label{eq:phi}
    \phi(\lambda, p) \coloneqq \big( e^\lambda - \lambda - 1 \big) p (1-p),
\end{equation}
and
\begin{equation}\label{eq:lambda}
    \lambda_* \coloneqq \argmax_{\lambda \in (0,\overline{\lambda}]} \min_{z \in \{z_*-\epsilon, z_*+\epsilon\} } \left[ \lambda \big| f(z) - f(z_*) \big| - \phi(\lambda, f(z)) \right],
\end{equation}
with $\overline{\lambda} \coloneqq \sup_{\lambda > 0} \{ \lambda: \lambda \geq e^\lambda - \lambda -1 \} \approx 1.2564$.
Given the target confidence level $\delta \in (0,1)$, the \texttt{FIT-Q} algorithm stops when this test statistic exceeds a predefined threshold, $\log(2/\delta)$:
\begin{equation} \label{eq:stopping-criteria}
    \tau = \inf\left\{ t \in \mathbb{N} \middle| Z_t^{\epsilon} \geq \log\left( \frac{2}{\delta} \right) \right\}.
\end{equation}

The test statistic $Z_t^\epsilon$ evaluates the current estimate $\hat{\theta}_t$ by quantifying the evidence against all competing alternatives outside a small neighborhood.
Specifically, the parameter $\theta$ in~\eqref{eq:test-stats} represents a potential alternative to $\hat\theta_t$.
Thus, the adjusted deviation term $\lambda_*| f(\hat{\theta}_t - X_s) - f(\theta - X_s)|$ in~\eqref{eq:test-stats} indicates the difference in fit between the estimate and the alternative for a given observation $X_s$, while the term $\phi(\lambda_*, f(\theta-X_s))$ in~\eqref{eq:test-stats} acts as a penalty based on the scaled variance corresponding to the alternative itself.
Accordingly, the summation in~\eqref{eq:test-stats} over all historical observations describes the cumulative evidence against a specific alternative~$\theta$, measuring how poorly it explains the data compared to $\hat\theta_t$, even after accounting for the associated noise.
Furthermore, it is easy to check that the minimum in~\eqref{eq:test-stats} over the two points $\{\hat\theta_t-\epsilon,\hat\theta_t+\epsilon\}$ is equivalent to the infimum over all alternatives $\theta$ outside the $\epsilon$-neighborhood (i.e., where $|\theta-\hat\theta_t|>\epsilon$). Note that this equivalence significantly reduces computational complexity, eliminating the need to solve an optimization problem involving an infimum operation over a continuous parameter set. Consequently, the test statistic $Z_t^\epsilon$ represents the cumulative evidence against the strongest possible alternative from outside this neighborhood. By showing that this evidence is substantial, we can certify that the true parameter $\theta_*$ lies in the $\epsilon$-neighborhood of the current estimate $\hat\theta_t$. We note that the value $\lambda_*$ in~\eqref{eq:test-stats} governs the trade-off between the aforementioned two terms and is chosen specifically to maximize the asymptotic discriminative power of the test statistic.

In Section \ref{section:analysis}, we show that the \texttt{FIT-Q} algorithm with this stopping rule guarantees that the failure probability does not exceed the target error rate (i.e., $\mathbb{P}_{\theta_*}^{\pi_{\epsilon, \delta}} \big( | \hat{\theta}_\tau - \theta_* | > \epsilon \big) \leq \delta$), and it achieves the optimality in terms of the asymptotic growth rate of the expected stopping time, $\lim_{\delta \rightarrow 0} \mathbb{E}_{\theta_*}^{\pi_{\epsilon,\delta}}[\tau] / \log(1/\delta)$.
We further discuss the general choice of the function $\phi$ and the constant $\lambda$, highlighting the design principles behind our suggestion.

\paragraph{Prediction rule.}
The \texttt{FIT-Q} algorithm simply reports the current MME value at the end of time horizon in the FB setting (i.e., $\hat{\theta}_T$), or at the stopped moment in the FC setting (i.e., $\hat{\theta}_\tau$).

\subsection{Remark 1: Performance of Static Query Rule} \label{subsec:static-query}

We analyze the performance of static and constant query rules and represent the target performance metrics in terms of Fisher information.
Through this analysis, we justify the need of tracking the hindsight Fisher-optimal query.

Let $\pi^\texttt{static}(x)$ denote the algorithm that keeps asking the same question $x \in \mathcal{X}$, i.e., $X_t = x$ for all $t$, and adopts the MME for prediction.
Since the response function $f(\cdot)$ is monotonic and the MME $\hat{\theta}_t$ satisfies the moment condition $\sum_{s=1}^t f(\hat{\theta}_t-x) = \sum_{s=1}^t Y_s$, we have
    \begin{align}
        \mathbb{P}^{\pi^\texttt{static}(x)}_{\theta_*} \Big( \hat{\theta}_t \geq \theta_* + \epsilon\Big) 
        &= \mathbb{P}^{\pi^\texttt{static}(x)}_{\theta_*} \left( \sum_{s=1}^{t}  f(\hat{\theta}_t - x) \geq \sum_{s=1}^{t}f(\theta_* - x + \epsilon) \right)
	\\&= \mathbb{P}^{\pi^\texttt{static}(x)}_{\theta_*} \left( \frac{1}{t}\sum_{s=1}^{t} Y_s \geq f(\theta_* - x + \epsilon)\right).
    \end{align}
Also note that the responses $Y_1,\ldots,Y_t$ are i.i.d. Bernoulli random variables. This leads to the following limit result derived from the classical large deviations theory:\footnote{We refer to Sanov's theorem which gives large deviation results of i.i.d. random variables. When random variables follow the Bernoulli distribution, the rate function is given in the form of KL-divergence.}
\begin{equation}
    \lim_{t\to\infty}\frac1t\log\mathbb{P}^{\pi^\texttt{static}(x)}_{\theta_*} \Big(\hat{\theta}_t \geq \theta_* + \epsilon\Big) 
    = -d\big( f(\theta_* - x + \epsilon) \mid f(\theta_* - x) \big).
\end{equation}
A similar result can be shown for $\mathbb{P}^{\pi^\texttt{static}(x)}_{\theta_*} \Big( \hat{\theta}_T \leq \theta_* - \epsilon\Big)$, and thus, we have the following approximation:
\begin{equation}
    \mathbb{P}^{\pi^\texttt{static}(x)}_{\theta_*} \Big( \big| \hat{\theta}_t - \theta_* \big| \geq \epsilon \Big) 
    \approx \exp\left( - t \min \left\{ d\big( f(\theta_* - x + \epsilon) \mid f(\theta_* - x) \big), d\big( f(\theta_* - x - \epsilon) \mid f(\theta_* - x) \big) \right\} \right).
\end{equation}
On the other hand, the Fisher information arises from the second order approximation of the KL divergence with respect to $\epsilon$:
\begin{equation}
    d\big( f(\theta_* - x + \epsilon) \mid f(\theta_* - x) \big) \approx d\big( f(\theta_* - x - \epsilon) \mid f(\theta_* - x) \big) \approx\frac{I(x; \theta_* )}{2} \epsilon^2.
\end{equation}

\paragraph{Asymptotic decay rate of the failure probability.}
Let us first focus on the fixed-budget setting.
Then, combining the above results, we obtain
\begin{equation}
    \lim_{T \rightarrow \infty} \left\{-\frac{1}{T} \log \mathbb{P}^{\pi^\texttt{static}(x)}_{\theta_*} \Big( \big| \hat{\theta}_T - \theta_* \big| \geq \epsilon\Big)\right\}
    \approx \frac{I(x; \theta_* )}{2} \epsilon^2.
\end{equation}
This implies that the failure probability associated with the algorithm $\pi^\texttt{static}(x)$ decays exponentially, where the asymptotic decay rate is proportional to the Fisher information $I(x; \theta_*)$ of the question $x$.
Consequently, the ideal static query rule is the one asking the hindsight Fisher-optimal query, $x_* = \argmax_{x \in \mathbb{R}} I(x; \theta_*)$.
The \texttt{FIT-Q} algorithm, designed to track this query, is thus naturally justified.
In Section \ref{section:analysis}, we demonstrate that the \texttt{FIT-Q} algorithm indeed achieves the hindsight optimal performance.

\paragraph{Asymptotic growth rate of the required number of questions.}
We now focus on the fixed-confidence setting.
Consider a deterministic stopping rule\footnote{
    More formally, under the stopping rule $\tau = \lceil \left( \max\{ d(f(\theta_*-x + \epsilon)||f(\theta_* -x)), d(f(\theta_* - x - \epsilon)||f(\theta_* - x)) \} \right)^{-1} \log(2/\delta)\rceil$, it can be shown that the confidence requirement is met by the Chernoff-Hoeffding theorem.
} such that
\begin{equation}
    \tau \approx \frac{ 2\log(1/\delta) }{ I(x;\theta_*)\epsilon^2 }.
\end{equation}
If $\tau$ is large enough, this algorithm approximately satisfies the confidence constraint, i.e.,
\begin{equation}
    \mathbb{P}^{\pi^\texttt{static}(x)}_{\theta_*} \Big( \big| \hat{\theta}_\tau - \theta_* \big| \geq \epsilon \Big) 
    \approx \exp\left( - \frac{I(x; \theta_* )}{2} \epsilon^2\tau \right)
    \approx \delta.
\end{equation}
This result shows that the number of questions needed to satisfy the confidence requirement is inversely proportional to the Fisher information of the question $x$.
Again, the ideal static query rule is the one asking the hindsight Fisher-optimal query, thereby providing a clear justification for the \texttt{FIT-Q} algorithm.

\subsection{Remark 2: Comparison to Robbins-Monro Algorithm} \label{subsec:RM}

In this section, we compare our \texttt{FIT-Q} algorithm with the Robbins-Monro (\texttt{RM}) algorithm, a seminal method for solving stochastic root-finding problems.
Recall that $z_*$ is the maximizer of the function $f'(z)^2/ \left\{ f(z) \big( 1 - f(z) \big)\right\}$.
With $p_* \coloneqq f(z_*)$, finding the Fisher-optimal query $x_*$ is identical to a root-finding problem to find the value $x \in \mathbb{R}$ satisfying $\mathbb{E}[ Y_t | X_t = x ] = p_*$.
Applying the \texttt{RM} algorithm yields the following query rule:
\begin{equation}\label{eq:RM updating}
	X_{t+1} = X_t + \alpha_t \big( Y_t - p_* \big),
\end{equation}
where $\alpha_t$ is the step size.
Theorem 5.2 in \cite{pasupathy2011stochastic} implies that the chosen queries are normally distributed asymptotically, i.e., $t^{1/2}(X_t - x_*) \xrightarrow{d} N(0, V_\alpha(x_*))$, as $t$ grows, and the asymptotic variance term $V_\alpha(x_*)$ is minimized when $\alpha_t = 1/(t f'(z_*))$.

The \texttt{FIT-Q} algorithm behaves very similarly to the \texttt{RM} algorithm with optimal step sizes.
Recall that the MME $\hat{\theta}_t$ satisfies the moment condition, $\sum_{s=1}^t f(\hat{\theta}_t - X_s) = \sum_{s=1}^t Y_s$, and our query rule chooses a query $X_t$ such that $f(\hat{\theta}_{t-1} - X_t) = p_*$.
Then, the following identity holds
\begin{equation}
    Y_{t} - p_* = Y_t - f(\hat{\theta}_{t-1} - X_{t})
        = \sum_{s=1}^{t} \left\{ f(\hat{\theta}_{t} - X_s) - f(\hat{\theta}_{t-1} - X_s) \right\}.
\end{equation}
A first-order approximation of the right-hand side gives
\begin{equation}
    \sum_{s=1}^{t} \left\{ f(\hat{\theta}_{t} - X_s) - f(\hat{\theta}_{t-1} - X_s) \right\} 
    \approx (\hat{\theta}_{t} - \hat{\theta}_{t-1} ) \sum_{s=1}^{t} f'(\hat{\theta}_t - X_s) 
    = (X_{t+1} - X_t ) \sum_{s=1}^{t} f'(\hat{\theta}_t - X_s) ,
\end{equation}
where the last equality utilizes the fact that $\hat{\theta}_{t} - \hat{\theta}_{t-1} = (X_{t+1} + z_*) - (X_t + z_*) = X_{t+1} - X_t$.
One can also show that $t^{-1} \sum_{s=1}^{t} f'(\hat{\theta}_t - X_s)$ converges to $f'(z_*)$ as $t$ increases, which implies the following approximation:
\begin{equation}
    Y_t - p_* \approx t f'(z_*) (X_{t+1}-X_t).
\end{equation}
If the approximation is exact, this equation precisely matches the update rule of \texttt{RM} algorithm \eqref{eq:RM updating} with the optimal step size $\alpha_t = 1/(tf'(z_*))$.
That is, the asymptotic behaviors of the two algorithms are equivalent in the large-$t$ regime.

However, the inherent decoupling between the parameter estimate $\hat\theta_t$ and the chosen query $X_t$ under the \texttt{RM} algorithm leads to a significant challenge in analyzing their convergence rates, which are essential for assessing the algorithm's optimality. As shown in~\eqref{eq:RM updating}, the \texttt{RM} algorithm's query rule relies exclusively on the latest response $Y_t$, rather than on the estimate $\hat\theta_t$. This implies that the proximity of $\hat\theta_t$ to its true parameter $\theta_*$ does not necessarily ensure the closeness of $X_t$ to the optimal query $x_*$. Such a construction makes it difficult to link the behaviors of the two components and formally characterize the convergence rate of the algorithm. Although several studies have tackled this issue in the literature~\citep{chandak2022concentration, rahmani2016exponential, woodroofe1972normal}, the conditions required by their analyses are often hard to verify in practice, and the resulting convergence bounds are not sufficiently tight.

In constrast, our \texttt{FIT-Q} algorithm is designed to overcome this challenge. In particular,  our query rule~\eqref{eq:projection} intrinsically ties the estimate $\hat{\theta}_t$ and the query $X_t$ through a constant $z_*$. Consequently, as we shall see in the next section, we rigorously demonstrate that both $\hat{\theta}_t$ and $X_t$ achieve optimal convergence rates based on a finite-time error bound of the MME.

\section{Theoretic Analysis}\label{section:analysis}

In this section, we provide formal, theoretical analysis showing the asymptotic optimality of the \texttt{FIT-Q} algorithm.
The complete proofs for all results established here can be found in Appendices \ref{appendix:proof-FB} and \ref{appendix:proof-FC}.

\subsection{Optimality of \texttt{FIT-Q} in Fixed-Budget Setting} \label{subsec:FB-analysis}

Recall that in the fixed-budget (FB) setting, we use the failure probability as a performance metric:
\begin{equation}
	p^\text{FB}(\pi_{\epsilon, T}; \theta_*, \epsilon, T) \coloneqq \mathbb{P}_{\theta_*}^{\pi_{\epsilon, T}} \big( | \hat{\theta}_T - \theta_* | > \epsilon \big),
\end{equation}
To clarify our notion of optimality in the FB setting, we first define a class of competing algorithms.

\begin{definition}[FB-consistent algorithms]\label{definition:FB-consistency}
    An algorithm $\pi$ is said to be \emph{FB-consistent} if, for any given $\theta_* \in \mathbb{R}$ and $\epsilon \in \mathbb{R}_+$, the failure probability converges to zero as $T$ grows:
    \begin{equation}
        \lim_{T \rightarrow \infty} p^\text{FB}(\pi_{\epsilon,T}; \theta_*, \epsilon, T) = 0.
    \end{equation}
\end{definition}

This is a natural condition that any reasonable algorithm should satisfy.
Indeed, it can be shown that any algorithm, including \texttt{FIT-Q}, that adopts MME or MLE in its prediction rule is FB-consistent regardless of its query rule.

\begin{theorem}[Universal performance limit in the FB setting]\label{thm:FB-bound}
    For any FB-consistent algorithm $\pi$, the exponential decay rate of its failure probability cannot exceed $\frac{I(x_*;\theta_*)}{2} \epsilon^2$ asymptotically as $T$ grows; specifically,
	\begin{equation}
		\limsup_{T \to \infty} -\frac{1}{T} \log p^\text{FB}(\pi_{\epsilon,T}; \theta_*, \epsilon, T) \leq \frac{I(x_*; \theta_*)}{2} \epsilon^2 + o(\epsilon^2).
	\end{equation}
\end{theorem}

Theorem \ref{thm:FB-bound} establishes a limit on the achievable performance of any FB-consistent algorithm in the form of an asymptotic upper bound on the exponential decay rate.
The upper bound is characterized by the maximum Fisher-information $I(x_*; \theta_*)$, which is consistent with the argument made with respect to the static query rules in Section~\ref{subsec:static-query}.
The proof of this result is based on an information-theoretic argument, following a similar structure to the proof of lemma 2 in \cite{bassamboo2023learning}.

Our next theorem shows the FB-consistency of the \texttt{FIT-Q} algorithm and establishes an asymptotic lower bound on the exponential decay rate of \texttt{FIT-Q}'s failure probability.

\begin{theorem}[Performance of \texttt{FIT-Q} in FB setting]\label{thm:FIT-Q-FB}
    The \texttt{FIT-Q} algorithm is FB-consistent. Furthermore, if $x_* = \theta_* - z_* \in \mathcal{X}$, the asymptotic decay rate of \texttt{FIT-Q}'s failure probability satisfies:
	\begin{equation}
		\liminf_{ T \to \infty} - \frac{1}{T} \log p^\text{FB}(\pi^\text{FIT-Q}_{\epsilon,T}; \theta_*, \epsilon, T) \geq \frac{I(x_*; \theta_*)}{2} \epsilon^2 + o(\epsilon^2).
	\end{equation}
\end{theorem}

\begin{proof}[Proof sketch of Theorem \ref{thm:FIT-Q-FB}]
    Define $I_* \coloneqq I(x_*; \theta_*)$.
    We carefully choose a time moment $t_0$.
    By utilizing Ville's inequality, it can be shown that $|\hat{\theta}_t - \theta_*| \leq \sqrt{\epsilon}$ for all $t \geq t_0$ with probability at least $1 - 2 \exp( - \epsilon^2 TI_*/2)$.
    On this event, the \texttt{FIT-Q} algorithm after time $t_0$ will always ask questions that are $\sqrt{\epsilon}$-close to the hindsight Fisher-optimal query, i.e., $|X_t - x_*| \leq \sqrt{\epsilon}$ for all $t \geq t_0+1$.
    Accordingly, the distribution of the estimate $\hat{\theta}_t$ concentrates to the true value $\theta_*$ at a nearly optimal rate after time $t_0$.
    We obtain the desired result by utilizing the fact that $|\hat{\theta}_T - \theta_*| \leq \epsilon$ with probability at least $1 - \exp\big( - (\epsilon^2 +o(\epsilon^2))(T-t_0) I_*/2 \big)$ on the event.
\end{proof}

Observe that the asymptotic performance lower bound in Theorem~\ref{thm:FIT-Q-FB} matches the universal performance limit in Theorem~\ref{thm:FB-bound}. This leads to the conclusion that the \texttt{FIT-Q} algorithm achieves the optimal decay rate, ignoring $o(\epsilon^2)$ factors.

Although this result may appear to be a trivial consequence of concentration inequalities, the key analytical challenge lies in avoiding circular reasoning: in order for the estimate to concentrate quickly, near-optimal questions must be asked, and conversely, a sufficiently concentrated estimate is required to ask near-optimal questions. Such an interdependent behavior has been a well-recognized and persistent issue in the best-arm identification literature \citep{wang2024best}.

Our proof resolves this issue by separating the initial phase of concentration from the long-run phase of concentration: we show that the concentration during the initial phase is sufficient to accelerate the long-run concentration rate to its (near-)maximal level.
Ville's inequality (see Lemma \ref{lemma:Ville}) is utilized to derive a time-uniform deviation bound of the estimate process, playing a key role in the proof.
Our proof also exploits the fact that the \texttt{FIT-Q} algorithm maintains a constant gap between the current estimate and query.
This structure enables the concentration bound of the estimate process to be applied to bound the query process as well.

\subsection{Optimality of \texttt{FIT-Q} in Fixed-Confidence Setting} \label{subsec:FC-analysis}

In the fixed-confidence (FC) setting, we focus on the expected stopping time, $\mathbb{E}[\tau]$.
Analogous to the FB-consistency introduced for the fixed-budget setting, we begin by introducing a notion of consistency for the FC setting.

\begin{definition}[FC-consistent algorithms]\label{definition:FC-consistency}
    An algorithm $\pi$ is said to be \emph{FC-consistent} if it satisfies the following conditions for any $\theta_* \in \mathbb{R}$, $\epsilon \in \mathbb{R}_+$ and $\delta \in \mathbb{R}_+$.
    \begin{enumerate}[label=\textbf{(F\alph*)}, ref=(F\alph*), align=left, leftmargin=*]
        \item \label{cond:FC-defn-termination}
            Its stopping time $\tau$ is finite almost surely.
        \item \label{cond:FC-defn-correctness}
            Its failure probability does not exceed $\delta$, i.e., $\mathbb{P}_{\theta_*}^{\pi_{\epsilon,\delta}} \big( | \hat{\theta}_\tau - \theta_* | > \epsilon \big) \leq \delta$.
    \item \label{cond:FC-defn-stability}
        For any $\eta \in \mathbb{R}_+$, $\lim_{\delta \to 0} \mathbb{P}_{\theta_*}^{\pi_{\epsilon, \delta}} \big( | \hat{\theta}_\tau - \theta_* | > \eta \big) = 0$.
    \end{enumerate}
\end{definition}

Conditions~\ref{cond:FC-defn-termination} and \ref{cond:FC-defn-correctness} are necessary by the problem setup, while condition \ref{cond:FC-defn-stability} is a subtle, yet natural, regularity condition. It requires that as the confidence level becomes stricter (i.e., $\delta \to 0$), the distribution of the final prediction $\hat{\theta}_\tau$ not only satisfies the $\epsilon$-accuracy but also increasingly concentrates around the true value $\theta_*$. This requirement is not overly restrictive. For instance, the condition is naturally satisfied by algorithms coupling a consistent estimator (e.g., MME, MLE) with a stopping rule that ensures an unbounded increase of the sample size $\tau$ as the target failure probability $\delta$ approaches zero.

The following theorem characterizes a universal performance limit on the sample complexity for any FC-consistent algorithm.

\begin{theorem}[Universal performance limit in FC setting]\label{thm:FC-bound}
    For any FC-consistent algorithm $\pi$, the growth rate of its expected stopping time satisfies:
    \begin{equation}
		\liminf_{\delta \to 0}\frac{\mathbb{E}_{\theta_*}^{\pi_{\epsilon,\delta}} \left[ \tau \right]}{ \log (1/\delta)} \geq \left(  \frac{I(x_*;\theta_*)}{2}\epsilon^2 + o(\epsilon^2) \right)^{-1}.
    \end{equation}
\end{theorem}

This result is consistent with Theorem \ref{thm:FB-bound}; since the failure probability decays exponentially at most at a rate of $\frac{I(x_*;\theta_*)}{2}\epsilon^2$, it requires at least $\big(\frac{I(x_*;\theta_*)}{2}\epsilon^2\big)^{-1} \log(1/\delta)$ samples to meet the failure probability constraint.
The proof is based on an information-theoretic argument that extends techniques from the discrete decision space literature~\citep[e.g.,][]{bassamboo2023learning, kaufmann2016complexity} to our continuous setting. A key challenge in this extension to the continuous space is that, due to the overlapping nature of ``success intervals'', an algorithm can be technically $(\epsilon,\delta)$-correct for two hypotheses that are themselves separated by only $\epsilon$, without being able to distinguish between them. The FC-consistency, through condition \ref{cond:FC-defn-stability}, resolves this ambiguity by requiring convergence to the exact true parameter, which ensures that close hypotheses become distinguishable in the limit. This allows the information-theoretic machinery to yield a tight lower bound.

We now turn our attention to the performance of \texttt{FIT-Q} algorithm.
Recall that it adopts the following test statistics:
\begin{equation}
    Z_t^{\epsilon} \coloneqq \min_{\theta \in \{\hat{\theta}_t-\epsilon, \hat{\theta}_t+\epsilon\}} \sum_{s=1}^t \left[ \lambda_* \left| f(\hat{\theta}_t - X_s) - f(\theta - X_s) \right| - \phi( \lambda_*, f(\theta - X_s) ) \right].
\end{equation}
Although we have made specific suggestions for the choice of $\phi:\mathbb{R}_+ \times [0,1] \rightarrow \mathbb{R}_+$ and $\lambda_* \in \mathbb{R}_+$ as in \eqref{eq:phi} and \eqref{eq:lambda}, respectively, in Section \ref{section:algorithm}, here we provide general conditions for $\phi$ and $\lambda_*$ to ensure the proposed algorithm's FC-consistency (Theorem \ref{thm:FIT-Q-FC-consistency}) and optimality (Theorem \ref{thm:FIT-Q-FC-optimality}).

\begin{theorem}[FC-consistency of \texttt{FIT-Q}] \label{thm:FIT-Q-FC-consistency}
    Assume that the function $\phi(\lambda_*, \cdot)$, when evaluated with a suitable scaling $\lambda_* = \lambda_*(\epsilon)$, satisfies the following conditions: for any $\epsilon > 0$,
    \begin{enumerate}[label=\textbf{(C\alph*)}, ref=(C\alph*), align=left, leftmargin=*]
        \item \label{cond:FC-cons-quasi-convex} For any $p_0 \in [0,1]$, the function $p \mapsto \lambda_* | p - p_0 | - \phi(\lambda_*, p)$ is continuous and quasi-convex, and it is minimized at $p = p_0$.
        \item \label{cond:FC-cons-cgf-dominate} For any $p\in[0,1]$, $\phi(\lambda_*, p) \geq \max\{ \psi(\lambda_*, p), \psi(-\lambda_*, p) \}$, where $\psi(\lambda, p) \coloneqq \log \mathbb{E}[ \exp(\lambda (Y-p)) ]$ with ${Y \sim \text{Bernoulli}(p)}$.
        \item \label{cond:FC-cons-positive} The value $\min_{z\in\{z_*-\epsilon, z_* +\epsilon\}}[\lambda_* | f(z) - f(z_*) | - \phi( \lambda_*, f(z) ) ]$ is strictly positive.
    \end{enumerate}
    Then, the \texttt{FIT-Q} algorithm with the stopping time $\tau \coloneqq \inf\{ t \in \mathbb{N} : Z_t^\epsilon \geq \log(2/\delta)\}$ is \emph{FC-consistent}.
\end{theorem}

The conditions in Theorem~\ref{thm:FIT-Q-FC-consistency} are not arbitrary but embody crucial design principles for the proposed algorithm. The quasi-convex structure imposed by condition \ref{cond:FC-cons-quasi-convex} ensures computational tractability by guaranteeing that the most challenging alternative hypothesis lies at the boundary points $\hat{\theta}_t \pm \epsilon$, which reduces an otherwise intractable search problem to a simple comparison. The formal correctness guarantee, $\mathbb{P}_{\theta_*}^{\pi_{\epsilon,\delta}} \left( | \hat{\theta}_\tau - \theta_* | > \epsilon \right) \leq \delta$, is established through condition \ref{cond:FC-cons-cgf-dominate}, which enables the construction of a test supermartingale by ensuring that the penalty function $\phi$ conservatively bounds the cumulant generating function $\psi$; this is a prerequisite for applying Ville's inequality. Lastly, condition \ref{cond:FC-cons-positive} guarantees the termination of the algorithm by enforcing a strictly positive drift in the test statistic, ensuring that it will make forward progress and almost surely surpass any finite stopping threshold.

\begin{proof}[Proof sketch of Theorem \ref{thm:FIT-Q-FC-consistency}]
    We now sketch the proof that the \texttt{FIT-Q} algorithm satisfies conditions \ref{cond:FC-defn-termination}--\ref{cond:FC-defn-stability} for the FC-consistency.
    First, we show the stopping time $\tau$ is finite almost surely.
    Since the MME $\hat{\theta}_t$ converges to $\theta_*$ almost surely, the query $X_t$ also converges to $x_*$.
    Consequently, the increment of the test statistic, $Z_t^{\epsilon} - Z_{t-1}^{\epsilon}$, converges almost surely to a limit.
    This positive drift ensures that $Z_t^\epsilon$ will eventually surpass any finite threshold, guaranteeing that $\tau$ is finite.
    Next, to show the $(\epsilon,\delta)$-correctness guarantee, we introduce the following test supermartingale:
    \begin{equation}
        M_t \coloneqq \exp\left( \lambda_* \sum_{s=1}^t (f(\theta_* - X_s) - Y_s) - \phi(\lambda_*, f(\theta_* - X_s)) \right).
    \end{equation}
    By condition \ref{cond:FC-cons-cgf-dominate}, $\{M_t\}_{t\geq 0}$ is a supermartingale with $M_0=1$.
    Conditions \ref{cond:FC-cons-quasi-convex} and \ref{cond:FC-cons-positive} imply that, if $\theta_* > \hat{\theta}_t + \epsilon$, then $Z_t^\epsilon \leq \log M_t$.
    Applying Ville's inequality yields $\mathbb{P}( \exists t : Z_t^\epsilon \geq \log(2/\delta), \theta_* \geq \hat{\theta}_t + \epsilon ) \leq \mathbb{P}( \exists t: \log M_t \geq \log(2/\delta) ) \leq \delta/2$.
    A similar argument for the other failure case bounds the total failure probability by $\delta$, thus satisfying condition \ref{cond:FC-defn-correctness}.
    Finally, the \texttt{FIT-Q}'s query rule and the MME do not depend on $\delta$.
    This allows us to consider the sequence of stopping times indexed by $\delta$, $\{ \tau_\delta \}_{\delta}$, sharing the same test statistics process $\{Z_t^\epsilon\}_{t \geq 0}$.
    Since the test statistic has a bounded increment, we have $\tau_\delta \to \infty$ almost surely as $\delta \to 0$.
    The strong consistency of the MME then ensures that the estimate at this stopping time, $\hat{\theta}_{\tau_\delta}$, converges to $\theta_*$ almost surely. This satisfies condition \ref{cond:FC-defn-stability}.
\end{proof}

We remark that while our analysis focuses on the complete \texttt{FIT-Q} algorithm, the core components of our proposed framework are modular. A careful reading of our proof reveals that when the MME is used with a test statistic and stopping rule satisfying the conditions of Theorem~\ref{thm:FIT-Q-FC-consistency}, both termination (condition~\ref{cond:FC-defn-termination}) and $(\epsilon, \delta)$-correctness (condition~\ref{cond:FC-defn-correctness}) are guaranteed, regardless of the specific query rule employed. Additionally, if the query rule does not depend on $\delta$ or $\epsilon$, like \texttt{FIT-Q}, condition~\ref{cond:FC-defn-stability} is also satisfied, making the entire procedure FC-consistent.

Having established the FC-consistency of \texttt{FIT-Q}, we now proceed to show its asymptotic optimality by analyzing its sample complexity.

\begin{theorem}[Performance of \texttt{FIT-Q} in FC setting] \label{thm:FIT-Q-FC-optimality}
    Suppose that the assumptions in Theorem \ref{thm:FIT-Q-FC-consistency} hold. In addition, we assume that 
    \begin{enumerate}[label=\textbf{(O\alph*)}, ref=(O\alph*), align=left, leftmargin=*]
        \item \label{cond:FC-opt-phi} $\sup_{p \in [0,1]}\left| \phi(\lambda, p) - \frac{\lambda^2}{2} p(1-p) \right| = o(\lambda^2)$;
        
        \item \label{cond:FC-opt-lambda} $\left| \lambda_*(\epsilon) - \epsilon \times f'(z_*) / (f(z_*)(1-f(z_*))) \right| = o(\epsilon)$.
    \end{enumerate}
    Then, if $x_* \in \mathcal{X}$, the growth rate of \texttt{FIT-Q}'s expected stopping time does not exceed $\left(\frac{I(x_*;\theta_*)}{2} \epsilon^2\right)^{-1}$ asymptotically as $\delta$ vanishes; specifically, 
    \begin{equation}
		\limsup_{\delta \to 0}\frac{\mathbb{E}_{\theta_*}^{\pi_{\epsilon,\delta}^\texttt{FIT-Q}} \left[ \tau \right]}{ \log (1/\delta)} \leq \left(  \frac{I(x_*;\theta_*)}{2}\epsilon^2 + o(\epsilon^2) \right)^{-1}.
    \end{equation}
\end{theorem}

The two additional conditions are technical requirements to ensure the asymptotic optimality. Intuitively, condition~\ref{cond:FC-opt-phi} requires our surrogate penalty function $\phi$ to be a sufficiently accurate second-order approximation of the true cumulant generating function, which is crucial for the final rate calculation to be correct. Condition~\ref{cond:FC-opt-lambda} ensures that the scaling parameter $\lambda_*$ is tuned optimally as the error margin $\epsilon$ shrinks, guaranteeing that the test statistic achieves its maximum possible discriminative power.

\begin{proof}[Proof sketch of Theorem \ref{thm:FIT-Q-FC-optimality}]
    Let $c(\lambda, \epsilon) \coloneqq \min_{z \in \{z_* - \epsilon, z_* + \epsilon\}} [ \lambda | f(z_*) - f(z) | - \phi( \lambda, f(z) ) ] > 0$.
    Since the increment of the test statistic converges to $c(\lambda_*, \epsilon)$, $\frac{\tau_\delta}{\log(1/\delta)}$ converges to $\frac{1}{c(\lambda_*, \epsilon)}$ almost surely.
    Through a careful decomposition, it can be shown that $\mathbb{E}[\tau_\delta^2] < \infty$, and hence, the process $\{ \frac{\tau_\delta}{\log(1/\delta)} \}_\delta$ is uniformly integrable.
    Therefore, $\frac{ \mathbb{E}[\tau_\delta] }{ \log(1/\delta) }$ converges to $\frac{1}{c(\lambda_*,\epsilon)}$.
    Conditions \ref{cond:FC-opt-phi} and \ref{cond:FC-opt-lambda} lead to $c(\lambda_*(\epsilon), \epsilon) = \frac{I(x_*;\theta_*)}{2} \epsilon^2 + o(\epsilon^2)$.
\end{proof}
 
Theorems~\ref{thm:FIT-Q-FC-consistency} and~\ref{thm:FIT-Q-FC-optimality} highlight that the \texttt{FIT-Q} algorithm achieves consistency and asymptotic optimality in the fixed-confidence setting, provided that $\phi$ and $\lambda_*$ satisfy a set of regularity conditions. 
In Section~\ref{section:algorithm}, we proposed a specific construction of $\phi$ and the corresponding scaling parameter $\lambda_*$. The following proposition formally states that this choice satisfies all theoretical requirements:
\begin{proposition} \label{prop:verification}
    Define $\phi$ and $\lambda_*$ as in \eqref{eq:phi} and \eqref{eq:lambda}, respectively,  
    Then, all conditions in Theorems~\ref{thm:FIT-Q-FC-consistency} and \ref{thm:FIT-Q-FC-optimality} are satisfied.
\end{proposition}

The proposed function $\phi(\lambda, p)$ in \eqref{eq:phi} is constructed as the tightest quadratic upper bound that simultaneously dominates both $\psi(\lambda, p)$ and $\psi(-\lambda, p)$ over $p \in [0,1]$, as illustrated in Figure~\ref{fig:phi}. 
The bound becomes tighter as $\lambda \searrow 0$, which ensures sharper control over the test statistic in a small-error regime.
The scaling parameter $\lambda_*$ in \eqref{eq:lambda} is selected to maximize the asymptotic growth rate of the test statistic under the \texttt{FIT-Q} algorithm.
Figure~\ref{fig:lambda} visualizes this objective: each dashed curve shows the objective value for a fixed $z \in \{z_* - \epsilon, z_* + \epsilon\}$, and the chosen $\lambda_*$ corresponds to the point that maximizes the minimum across the two.
Finally, the upper limit $\overline{\lambda}$ is imposed to ensure that $\phi$ satisfies condition~\ref{cond:FC-cons-cgf-dominate}.

\begin{figure}[h!]
    \centering
    \begin{subfigure}{0.5\linewidth}
        \includegraphics[width=\textwidth]{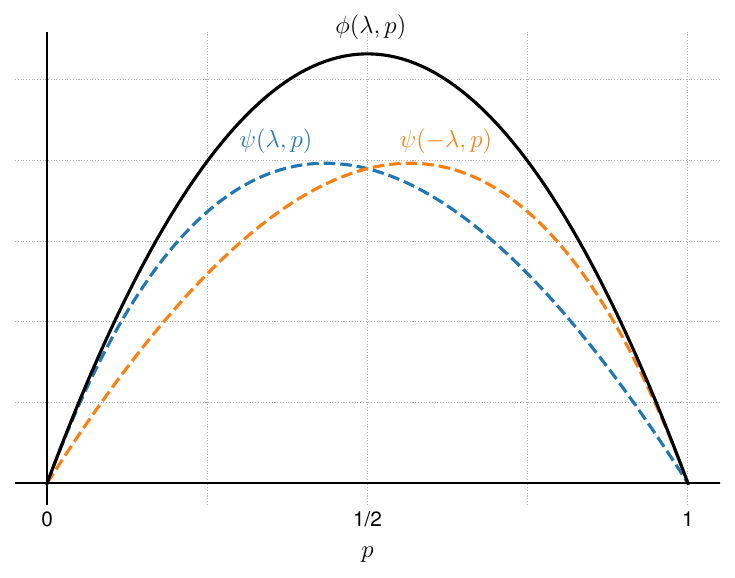}
        \caption{$\phi(\lambda, p)$ as an upper bound of $\psi(\lambda, p)$ and $\psi(-\lambda, p)$}
        \label{fig:phi}
    \end{subfigure}%
    \begin{subfigure}{0.5\linewidth}
        \includegraphics[width=\textwidth]{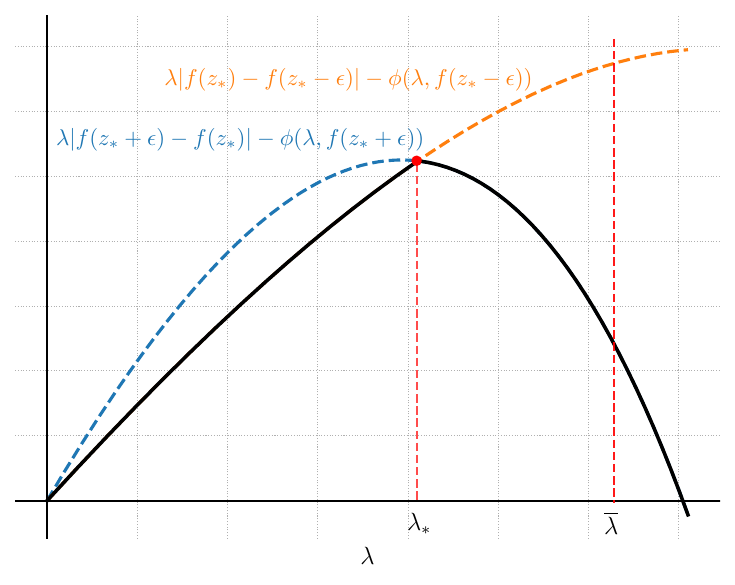}
        \caption{Objective curves for selecting $\lambda_*$ at $z = z_* \pm \epsilon$}
        \label{fig:lambda}
    \end{subfigure}
    \caption{Illustration of proposed $\phi$ and $\lambda_*$: (a) the proposed function $\phi(\lambda, p)$ (solid curve) and the log-moment generating functions $\psi(\lambda, p)$ and $\psi(-\lambda, p)$ (dashed curves), plotted over $p \in [0,1]$ for a fixed $\lambda > 0$; (b) the objective functions $\lambda \mapsto \lambda \cdot |f(z) - f(z_*)| - \phi(\lambda, f(z))$ for $z = z_* \pm \epsilon$ (dashed curves), used in the selection of $\lambda_*$.
}
\end{figure}

A key contribution of this work lies in developing a framework that natively addresses the continuous parameter and problem spaces, without leveraging common workarounds such as discretization. A straightforward alternative would be to discretize the continuous space, for instance by forming disjoint $2\epsilon$-intervals, and then apply test statistics from the discrete best-arm identification literature. However, such an approach can be sample-inefficient and requires ad-hoc choices for the discretization scheme. In contrast, our proposed test statistic is designed to operate directly on the original continuous problem. It provides a principled mechanism for accumulating evidence within the $\epsilon$-margin framework without altering the problem's structure, representing a more fundamental approach to sequential testing in continuous domains.

\section{Numerical Experiments}\label{section:experiment}

In this section, we verify our theoretical findings through numerical experiments.

\paragraph{Response models.}
We evaluate the proposed \texttt{FIT-Q} algorithm under two representative response models that satisfy Assumption~\ref{assumption:model} but exhibit distinct structural properties in terms of their Fisher information profiles:

\begin{itemize}
    \item \textbf{Logistic model}: The response function is defined as
        \begin{equation}
	   f(z) = \frac{1}{1 + \exp(-z)},
        \end{equation}
        which is commonly used in psychometrics and preference learning. Its fisher information function is given as
        \begin{equation}
            h(z) = f(z) \big(1 - f(z) \big) = \frac{\exp(-z)}{\big( 1 + \exp(-z) \big)^2},
        \end{equation}
        which is unimodal and symmetric, attaining a maximum at $z=z_*=0$, as shown in Figure \ref{fig:fisher-logistic}.
        Thus, the hindsight Fisher-optimal query is given by $x_* = \theta_*$, where the success probability is $f(z_*)=0.5$.
        
    \item \textbf{Algebraic-4 model}: The response function is given by
        \begin{equation}
    	f(z) = 0.5 \cdot \frac{z}{ \left( 1 + |z|^4 \right)^{1/4}}  + 0.5,
        \end{equation}
        which is smoother at the tails but exhibits bimodal Fisher information, as shown in Figure \ref{fig:fisher-algebraic4}.
        In this case, the Fisher information $h(z)$ attains its maximum at two symmetric points $z = \pm z_*$ where $z_* \approx 0.508$. The corresponding hindsight-optimal query is $x_* = \theta_* - z_*$, with the associated success probability approximately given by $f(z_*) \approx 0.75$.
\end{itemize}
Throughout all experiments, we fix \( \theta_* = 0 \) without loss of generality, leveraging the location invariance of the response model.

\begin{figure}[h!]
    \centering
    \begin{subfigure}{0.5\linewidth}
        \centering
        \includegraphics[width=\textwidth]{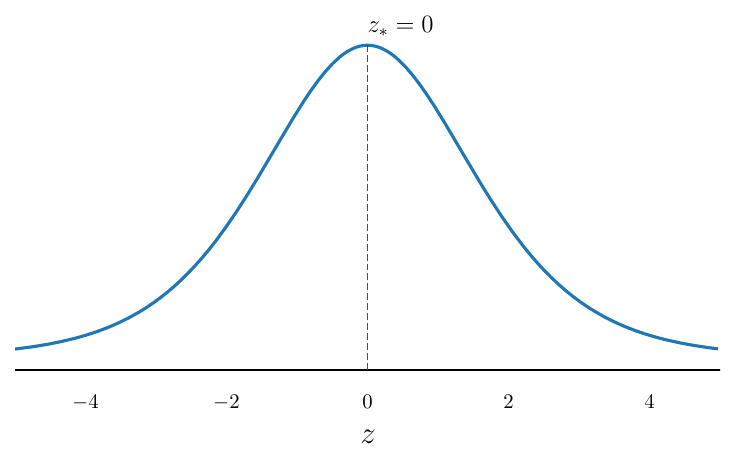}
        \caption{Logistic model}
        \label{fig:fisher-logistic}
    \end{subfigure}%
    \begin{subfigure}{0.5\linewidth}
        \centering
        \includegraphics[width=\textwidth]{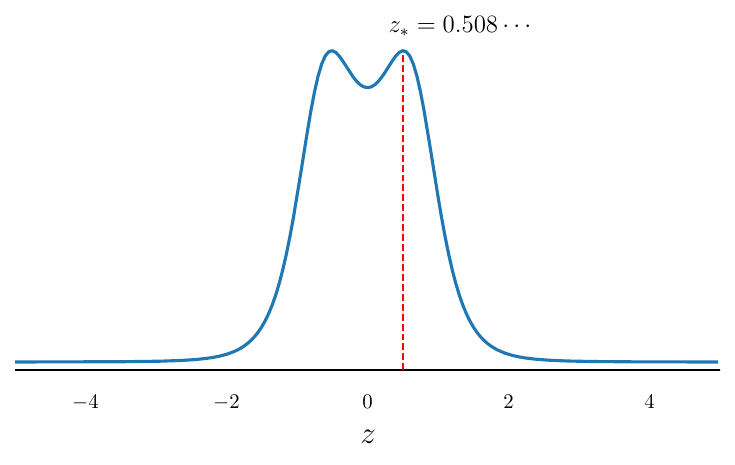}
        \caption{Algebraic-4 model}
        \label{fig:fisher-algebraic4}
    \end{subfigure}
        
    \caption{Fisher information function $h(z)$ for logistic and algebraic models.}
    \label{figure:fisher}
\end{figure}

\paragraph{Baseline algorithms.}
We compare the performance of \texttt{FIT-Q} against the following benchmark algorithms, each employing a distinct query strategy:

\begin{itemize}
    \item \textbf{Static query (optimal)}: The query is fixed at the hindsight Fisher-optimal point $x_* = \theta_* - z_*$, which is not implementable in practice but serves as an idealized performance upper bound. This algorithm is denoted by \texttt{Static($X_t = x_*$)}.
    
    \item \textbf{Static query (suboptimal)}: The query is fixed at a suboptimal point, specifically $x = 1$, which is typically far from the optimal query and results in low Fisher information. This algorithm is denoted by \texttt{Static($X_t = x$)}.

    \item \textbf{Uniform query}: At each step, the query is sampled independently and uniformly from the interval $(-1, 1)$, without adapting to past responses. This strategy serves as a randomized non-adaptive baseline. This algorithm is denoted by \texttt{Uniform(-1,1)}.

    \item \textbf{Robbins-Monro (RM)}: This algorithm applies the Robbins-Monro stochastic approximation method to iteratively solve for the query $x$ that achieves the target response probability $p_* = f(z_*)$. Specifically, the query is updated as
    \begin{equation}
        X_{t+1} = X_t + \alpha_t (Y_t - p_*),
    \end{equation}
    where the step size is set to $\alpha_t = 1 / (t f'(z_*))$, the asymptotically optimal rate under mild conditions.
    This algorithm is denoted by \texttt{Robbins-Monro}.
\end{itemize}

\paragraph{Fixed-budget setting.}  
We evaluate the accuracy of each algorithm under a fixed time budget \( T \), with the goal of minimizing the failure probability \( \mathbb{P}(|\hat{\theta}_T - \theta_*| > \epsilon) \). The error tolerance is fixed at \( \epsilon = 0.2 \), and we consider various time budgets: \( T \in \{200, 400, 600, 800, 1000\} \) for the logistic model, and \( T \in \{50, 100, 150, 200, 250\} \) for the algebraic model. For each configuration, we perform 100{,}000 independent runs to accurately estimate the empirical failure probability.

The results are summarized in Table~\ref{table:failure-combined}. In both models, \texttt{FIT-Q} consistently outperforms all adaptive and non-adaptive baselines, and closely matches the oracle performance of \texttt{Static(\( X_t = x_* \))}. As expected, \texttt{Static(\( X_t = 1 \))} performs poorly due to low Fisher information at the query point, and the \texttt{Uniform(-1,1)} strategy is consistently inferior to adaptive methods. Interestingly, the Robbins-Monro algorithm achieves competitive performance, particularly in the logistic model, where the Fisher information is unimodal.

\begin{table}[h!]
\centering
\begin{tabular}{llccccc}
\hline
\textbf{Model} & \textbf{Algorithm} & \textbf{$T=50/200$} & \textbf{$100/400$} & \textbf{$150/600$} & \textbf{$200/800$} & \textbf{$250/1000$} \\
\hline
\multirow{5}{*}{\textbf{Logistic}} 
& \texttt{FIT-Q} & 0.16417 & 0.04815 & 0.01516 & 0.00465 & 0.00149 \\
& \texttt{Static($X_t = x_*$)} & 0.17858 & 0.05174 & 0.01621 & 0.00516 & 0.00188 \\
& \texttt{Static($X_t = 1$)} & 0.20585 & 0.07176 & 0.02749 & 0.01269 & 0.00490 \\
& \texttt{Robbins-Monro} & 0.16507 & 0.04869 & 0.01516 & 0.00479 & 0.00150 \\
& \texttt{Uniform(-1,1)} & 0.17453 & 0.05584 & 0.01901 & 0.00676 & 0.00221 \\
\hline
\multirow{5}{*}{\textbf{Algebraic-4}} 
& \texttt{FIT-Q} & 0.16983 & 0.04306 & 0.01265 & 0.00386 & 0.00121 \\
& \texttt{Static($X_t = x_*$)} & 0.14761 & 0.03833 & 0.01106 & 0.00376 & 0.00100 \\
& \texttt{Static($X_t = 1$)} & 0.32834 & 0.11931 & 0.09963 & 0.04397 & 0.02003 \\
& \texttt{Robbins-Monro} & 0.17504 & 0.04487 & 0.01287 & 0.00358 & 0.00123 \\
& \texttt{Uniform(-1,1)} & 0.16618 & 0.04884 & 0.01612 & 0.00533 & 0.00159 \\
\hline
\end{tabular}
\caption{Failure probabilities of algorithms under fixed budgets. For the logistic model, $T \in \{200, 400, 600, 800, 1000\}$; for the algebraic-4 model, $T \in \{50, 100, 150, 200, 250\}$.}
\label{table:failure-combined}
\end{table}

To further examine asymptotic behavior, we plot the negative logarithm of failure probability, $-\log \mathbb{P}(|\hat{\theta}_T - \theta_*| > \epsilon)$, against $T$ in Figure~\ref{fig:log-failure}. 
All algorithms exhibit approximately linear decay, indicating exponential error reduction over time. 
Notably, \texttt{FIT-Q} aligns almost exactly with the slope predicted by Theorem~\ref{thm:FIT-Q-FB}, confirming that it achieves the optimal asymptotic decay rate up to leading order.

\begin{figure}[h!]
    \centering
    \begin{subfigure}[b]{0.495\textwidth}
        \centering
        \includegraphics[width=\textwidth]{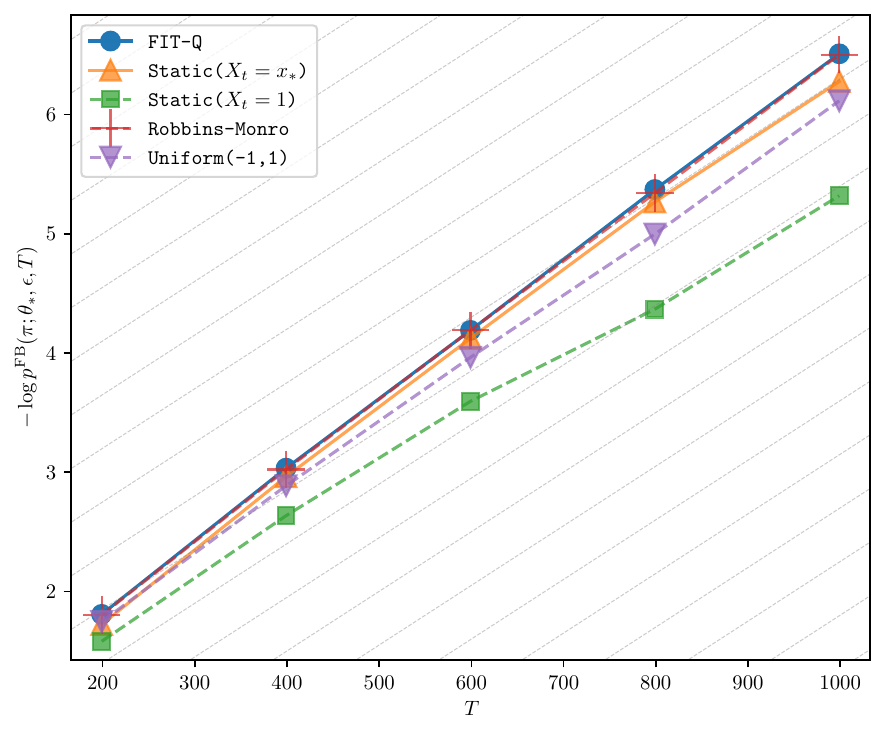}
        \caption{Logistic model}
        \label{fig:log-failure-logistic}
    \end{subfigure}
    \hfill
    \begin{subfigure}[b]{0.495\textwidth}
        \centering
        \includegraphics[width=\textwidth]{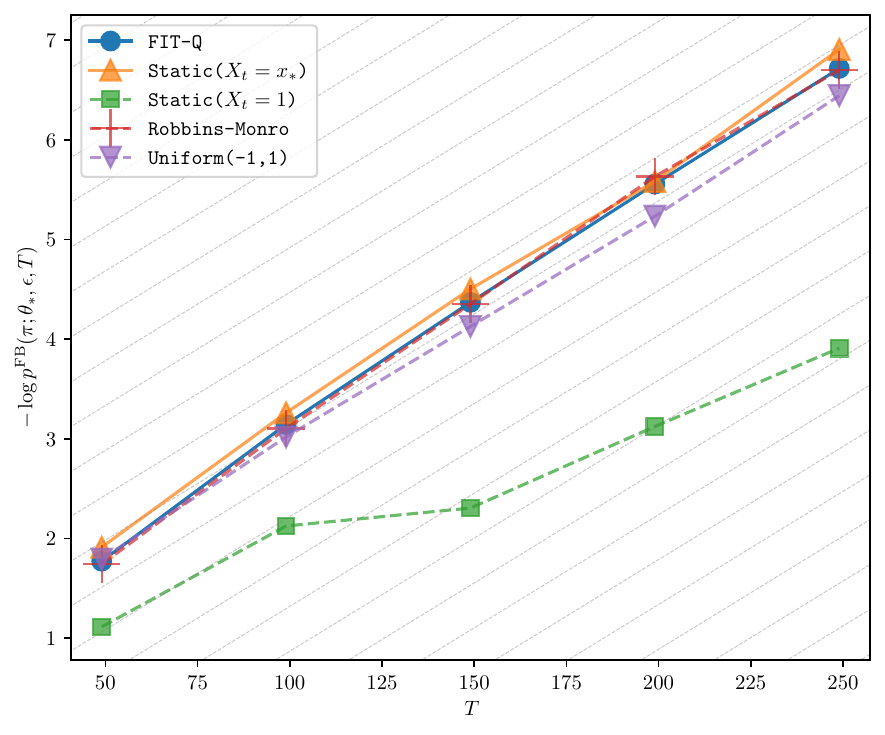}
        \caption{Algebraic-4 model}
        \label{fig:log-failure-algebraic4}
    \end{subfigure}
        \caption{Negative log failure probability under the fixed-budget setting for (a) logistic and (b) algebraic-4 models. The background diagonals correspond to the optimal slope $\frac{1}{2} I(x_*; \theta_*) \epsilon^2$ predicted by Theorem~\ref{thm:FB-bound}.}
    \label{fig:log-failure}
\end{figure}

\newpage
\paragraph{Fixed-confidence setting.}

We evaluate the performance of each algorithm in the fixed-confidence setting, where the goal is to estimate $\theta_*$ with tolerance $\epsilon = 0.2$ and failure probability at most $\delta$. 
We vary $\delta \in \{5 \times 10^{-2}, \dots, 5 \times 10^{-7} \}$, and estimate the expected stopping time $\mathbb{E}[\tau]$ via 3,000 independent replications per setting. 
\textit{All algorithms use the same stopping rule defined in~\eqref{eq:stopping-criteria}, applied to the test statistics introduced in~\eqref{eq:test-stats}.}

The results are reported in Table~\ref{table:stopping-combined}. 
In both models, FIT-Q consistently achieves the lowest expected stopping time among all implementable strategies, closely matching the oracle performance of \texttt{Static($X_t = x_*$)}. 
While the Robbins-Monro method also performs competitively, especially in the logistic case, it lags behind FIT-Q in the algebraic-4 model where the Fisher information is bimodal. 
Non-adaptive strategies such as \texttt{Static($X_t = 1$)} and \texttt{Uniform(-1,1)} incur significantly larger sample complexity across all confidence levels.

\begin{table}[h!]
\centering
\begin{tabular}{llcccccc}
\hline
\textbf{Model} & \textbf{Algorithm} & $\delta=5 \times 10^{-2}$ & $5 \times 10^{-3}$ & $5 \times 10^{-4}$ & $5 \times 10^{-5}$ & $5 \times 10^{-6}$ & $5 \times 10^{-7}$ \\
\hline
\multirow{5}{*}{\textbf{Logistic}} 
& \texttt{FIT-Q} & 791.5 & 1280.9 & 1770.2 & 2259.5 & 2748.4 & 3237.4 \\
& \texttt{Static($X_t = x_*$)} & 784.9 & 1273.8 & 1762.5 & 2251.4 & 2740.2 & 3229.1 \\
& \texttt{Static($X_t = 1$)} & 1002.7 & 1627.3 & 2250.5 & 2874.5 & 3498.6 & 4122.1 \\
& \texttt{Robbins-Monro} & 791.9 & 1281.3 & 1770.7 & 2259.8 & 2748.8 & 3237.8 \\
& \texttt{Uniform(-1,1)} & 849.3 & 1378.6 & 1907.9 & 2437.2 & 2966.5 & 3495.9 \\
\hline
\multirow{5}{*}{\textbf{Algebraic-4}} 
& \texttt{FIT-Q} & 212.2 & 339.9 & 467.9 & 595.5 & 723.2 & 850.8 \\
& \texttt{Static($X_t = x_*$)} & 206.0 & 333.2 & 460.6 & 588.2 & 715.4 & 842.9 \\
& \texttt{Static($X_t = 1$)} & 443.4 & 719.6 & 994.7 & 1265.9 & 1535.0 & 1815.8 \\
& \texttt{Robbins-Monro} & 211.3 & 338.8 & 466.5 & 594.0 & 721.5 & 849.1 \\
& \texttt{Uniform(-1,1)} & 218.1 & 352.9 & 487.8 & 622.6 & 757.5 & 892.3 \\
\hline
\end{tabular}
\caption{
Expected stopping times under the fixed-confidence setting. Each algorithm uses the stopping rule in~\eqref{eq:stopping-criteria} with test statistics from~\eqref{eq:test-stats}. 
For both models, we vary $\delta \in \{5 \times 10^{-2}, 5 \times 10^{-3}, \dots, 5 \times 10^{-7}\}$ with fixed tolerance $\epsilon = 0.2$.
}
\label{table:stopping-combined}
\end{table}

Figure~\ref{fig:stop} shows the relationship between $\mathbb{E}[\tau]$ and $\log(1/\delta)$, plotted on a log-linear scale. 
All curves exhibit approximately linear trends, indicating logarithmic dependence on $1/\delta$. 
Notably, the slope of the \texttt{FIT-Q} curve aligns closely with the inverse of the Fisher information at the optimal query, consistent with the theoretical lower bound in Theorem~\ref{thm:FC-bound}.

\begin{figure}[h!]
    \centering
    \begin{subfigure}[b]{0.49\textwidth}
        \centering
        \includegraphics[width=\textwidth]{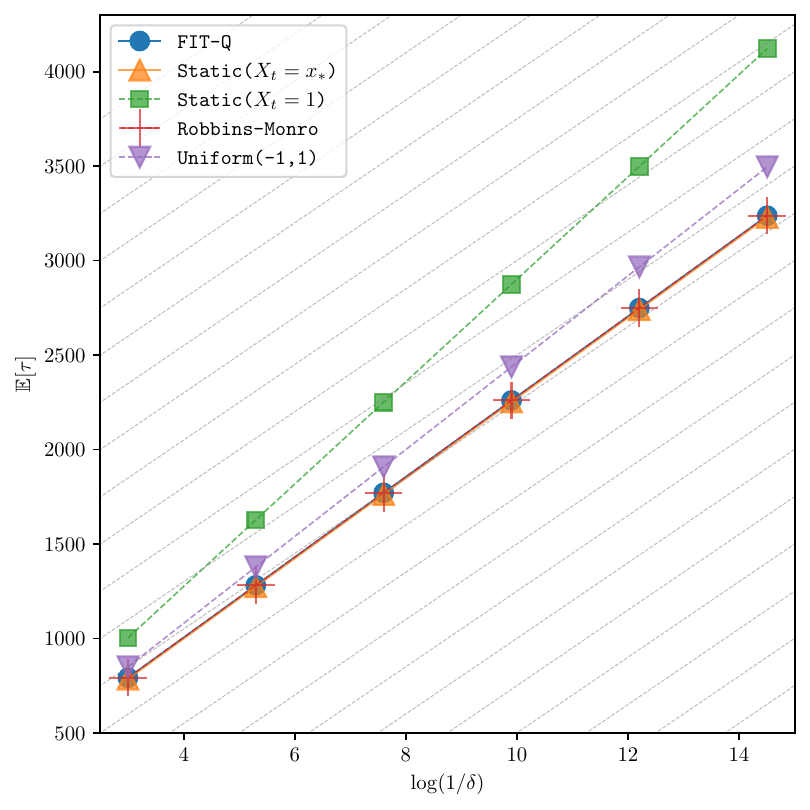}
        \caption{Logistic model}
        \label{figure:stop-logistic}
    \end{subfigure}
    \hfill
    \begin{subfigure}[b]{0.49\textwidth}
        \centering
        \includegraphics[width=\textwidth]{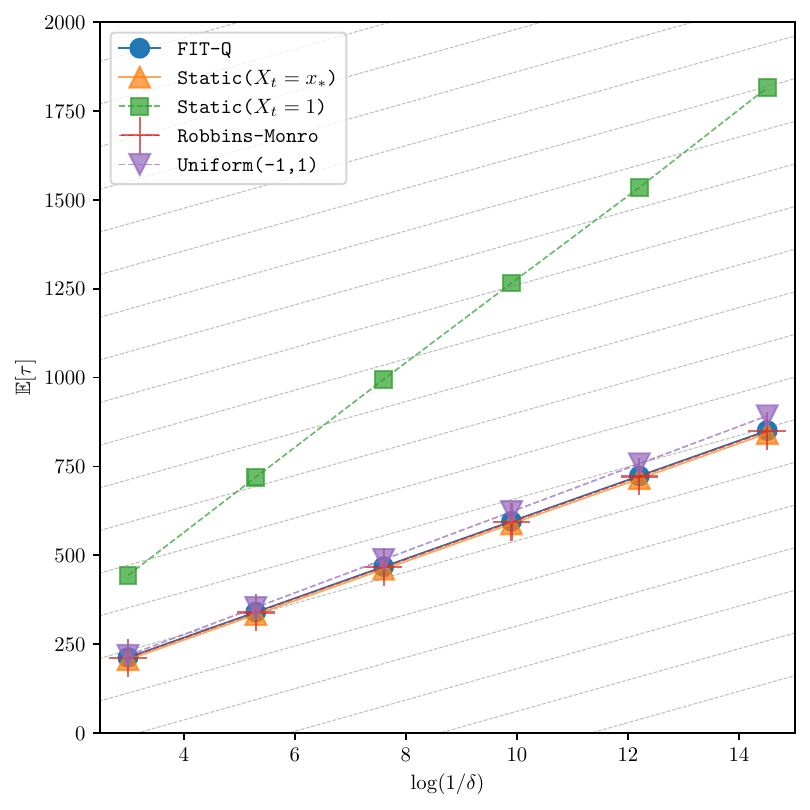}
        \caption{Algebraic4 model}
        \label{figure:stop-algebraic4}
    \end{subfigure}
        \caption{
Expected stopping times under the fixed-confidence setting, plotted against $\log(1/\delta)$ for (a) logistic and (b) algebraic-4 models. 
The background diagonals correspond to the theoretical slope predicted by Theorem~\ref{thm:FC-bound}.
}
    \label{fig:stop}
\end{figure}

\section{Conclusion}\label{section:conclusion}

In this paper, we studied the adaptive testing problem of identifying a candidate's ability parameter from sequential binary responses.
We proposed a simple and intuitive algorithm, \texttt{FIT-Q}, that adaptively selects questions to maximize Fisher information based on a method-of-moments estimate.
Our theoretical analysis demonstrated that this strategy, despite its simplicity, is information-theoretically justified and achieves asymptotic optimality in both fixed-budget and fixed-confidence settings, which was verified by our numerical experiments as well.

While our findings provide strong theoretical support for Fisher information tracking strategies, this work is not without its limitations and suggests several promising directions for future research.

First, our analysis relies on a response model that may be restrictive in some practical scenarios. We assume that the probability of a correct response is governed by a one-dimensional and shift-invariant function of the ability-difficulty gap, i.e., $f(\theta_*-x)$.
An intriguing extension would be to investigate a response function of the form $f({\boldsymbol\theta}_*, {\bf x})$, where both the ability ${\boldsymbol\theta}_*$ and the difficulty ${\bf x}$ are multivariate. This generalization, however, introduces significant theoretical complexities. In particular, when considering the multivariate case, where the Fisher information becomes a matrix, the criterion for optimal query selection---such as maximizing the determinant (D-optimality) or the minimum eigenvalue (E-optimality)---requires careful consideration. Beyond this extension, further research could analyze the sub-optimality of the \texttt{FIT-Q} algorithm under misspecified response functions or develop non-parametric methods relaxing structural assumptions on the response function.

Second, our primary theoretical guarantees are concerned with asymptotic optimality defined within a specific asymptotic regime.
Specifically, for the fixed-budget setting, this regime involves taking the limit as the budget $T \to \infty$ followed by the error margin $\epsilon \to 0$, while for the fixed-confidence setting, the limit is taken as the confidence level $\delta \to 0$ before $\epsilon \to 0$.
A more explicit characterization of this two-dimensional asymptotic behavior, or the derivation of performance bounds that hold uniformly for finite parameters, would constitute a valuable theoretical advancement.

Finally, our framework exclusively utilizes the method of moments estimator (MME).
Although MME offers computational and analytical convenience, other estimators, such as the maximum likelihood estimator (MLE), are widely used and possess desirable statistical properties.
An interesting open question, therefore, is whether the \texttt{FIT-Q} framework can be extended to a broader class of estimators and how the choice of estimator impacts the algorithm's overall efficiency and convergence guarantees.

\newpage
\appendix
\section*{Appendix}

\section{Preliminary Analysis}\label{appendix:lemmas}

\begin{lemma}\label{lemma:KL-approximation}
    The KL-divergence between $\text{Bernoulli}(f(\theta - x))$ and $\text{Bernoulli}(f(\theta - x +\epsilon))$ can be approximated by Fisher information with error of order $O(\epsilon^3)$ uniformly in $x \in \mathcal{X}$, i.e.,
    \begin{equation}
        d\big( f(\theta - x +\epsilon) | f(\theta - x) \big) = \frac{1}{2}I(x; \theta) \epsilon^2 + O(\epsilon^3),
    \end{equation}
    and
    \begin{equation}
        d\big( f(\theta - x) | f(\theta - x +\epsilon) \big) = \frac{1}{2}I(x; \theta) \epsilon^2 + O(\epsilon^3),
    \end{equation}
    where $O(\cdot)$ terms do not have dependence on $x$.
    Similarly, KL-divergence between $\text{Bernoulli}(f(\theta - x))$ and $\text{Bernoulli}(f(\theta -x-\epsilon))$ can be approximated with error of the same order uniformly in $x \in \mathcal{X}$.
\end{lemma}

\begin{proof}

Consider a KL-divergence between two Bernoulli distributions with probability $p$ and $q$, and treat it as a function about $p$, i.e.
\begin{equation}
    D_q(p) \coloneqq d( p | q) = p \log \frac{p}{q} + (1-p) \log \frac{1-p}{1-q}.
\end{equation}
Then, its derivatives are given as
\begin{gather}
    D_q'(p) = \log \frac{p}{q} - \log \frac{1-p}{1-q},\\
    D_q''(p) = \frac{1}{p} + \frac{1}{1-p},\\
    D_q'''(p) = - \frac{1}{p^2}+ \frac{1}{(1-p)^2}.
\end{gather}
By Taylor's theorem, for some $\xi$ between $p$ and $q$, 
\begin{equation}
    d(p|q) = \frac{(p-q)^2}{2 q(1-q)} + \frac{(p-q)^3}{6}D_q'''(\xi).
\end{equation}
For a fixed $\theta \in \mathbb{R}$, we have a compact problem set $\mathcal{X}$, $p$ and $q$ in our consideration lie in the compact interval $[p_\text{min}, p_\text{max}]$, where $0 < p_\text{min} < p_\text{max} < 1$.
Precisely, two thresholds are given as $ p_\text{min} \coloneqq \min_{x\in\mathcal{X}} f(\theta - x - \epsilon)$ and  $p_\text{max} \coloneqq \max_{x \in \mathcal{X}}f(\theta - x + \epsilon)]$, respectively.
Now, we obtain the bound $C$ on the value of $|D_q'''|$ 
\begin{equation}
    |D_q'''(p)| \leq \frac{1}{p^2} + \frac{1}{(1-p)^2} \leq \frac{1}{p_{\text{min}}^2} + \frac{1}{(1-p_{\max})^2} \eqqcolon C,
\end{equation}
which is independent of $p$.

Now, we claim that $ \Delta(x) \coloneqq |f(\theta - x + \epsilon) - f(\theta - x)|$ can be approximated as $f'(\theta - x) \epsilon$ with error of order $O(\epsilon^2)$ uniformly in $x \in \mathcal{X}$.
Note $\theta - x$ is contained in the compact interval $\theta - \mathcal{X} \coloneqq [\theta - x_R, \theta - x_L]$.
Since $f$ is thrice continuously differentiable, $f'$ and $f''$ are bounded on the compact set, i.e. $|f'| \leq M_1$ and $|f''| \leq 2M_2$.
Consequently, by Taylor's theorem, we can conclude that 
\begin{equation}
     \sup_{x \in \mathcal{X}}\big| \Delta(x) - f'(\theta - x) \epsilon \big| \leq  M_2\epsilon^2,
\end{equation}
i.e., $\Delta(x) = f'(\theta - x) \epsilon + O(\epsilon^2)$ uniformly.
Accordingly, we also get that $\Delta(x)^2 = f'(\theta - x)^2 \epsilon^2 + O(\epsilon^3)$ uniformly since $f'$ is bounded, i.e. $\sup_{x \in \mathcal{X}} | \Delta(x)^2 - f'(\theta - x)^2 \epsilon^2| \leq M_3 \epsilon^3$.
Combining the results together, we get 
\begin{align}
    \left| d\big(f (\theta - x + \epsilon) | f(\theta - x) \big) - \frac{I(x ; \theta)}{2} \epsilon^2 \right| \leq\; & \left| \frac{\Delta(x)^2}{2f(\theta- x) \big(1 - f(\theta-x)\big)} - \frac{f'(\theta - x)^2 \epsilon^2}{2f(\theta- x) \big(1 - f(\theta-x)\big)} \right| + \frac{C}{6} \Delta(x)^3\\
    \leq\; & \left| \frac{\Delta(x)^2- f'(\theta - x)^2 \epsilon^2}{2f(\theta- x) \big(1 - f(\theta-x)\big)} \right| + \frac{C}{6} \left( M_1 \epsilon + M_2 \epsilon^2 \right)^3\\
    \leq\; & \left| \frac{\Delta(x)^2- f'(\theta - x)^2 \epsilon^2}{2f(\theta- x) \big(1 - f(\theta-x)\big)} \right| + \frac{C}{6} \left( M_1 \epsilon + M_2 \epsilon^2 \right)^3\\
    \leq\; & \frac{M_3}{2f(\theta- x) \big(1 - f(\theta-x)\big)} \epsilon^3 + \frac{C}{6}(M_1 + M_2 \epsilon)^3 \epsilon^3,
\end{align}
which implies the desired result.

Similarly, if we define a KL-divergence between two Bernoulli distributions with probability $p$ and $q$ as a function about $q$, i.e., $D_p(q) \coloneq d(p |q)$, its derivatives are given as
\begin{gather}
    D_p'(q) = -\frac{p}{q} + \frac{1-p}{1-q},\\
    D_p''(q) = \frac{p}{q^2} + \frac{1-p}{(1-q)^2},\\
    D_p'''(q) = -2 \frac{p}{q^3} +2 \frac{1-p}{(1-q)^3}.
\end{gather}
By Taylor's theorem, for some $\xi$ between $p$ and $q$,
\begin{equation}
    d(p|q) = \frac{(p-q)^2}{2p(1-p)} + \frac{(p-q)^3}{6} D_q'''(\xi),
\end{equation}
and $D_q'''$ is bounded as 
\begin{equation}
    |D_p'''(q)| = \left| -2 \frac{p}{q^3} +2 \frac{1-p}{(1-q)^3} \right| \leq 2 \left|\frac{1}{q^3} + \frac{1}{(1-q)^3}\right| \leq \frac{2}{p_\text{min}^3} + \frac{2}{(1-p_\text{max})^3}
\end{equation}
regardless of the value of $q$.
Therefore, by the same reason above, we also get the desired result.

\end{proof}

\begin{proposition}[Strong consistency of MME]\label{prop:MME-consistency}
	Let $\hat{\theta}_t$ be method of moment estimator. Under any query rule, $\hat{\theta}_t$ is strongly consistent, i.e., $\hat{\theta}_t \to \theta_*$ almost surely as $t \to \infty$, hence $\hat{\theta}_t \to \theta_*$ in probability.	
\end{proposition}

\begin{proof}
We first consider the quantity $\sum_{s=1}^{t} [Y_s - f(\theta_* - X_s)]/t$. Since its numerator $\sum_{s=1}^{t} [ Y_s - f(\theta_* - X_s)]$ is a martingale with bounded increments, we have
\begin{equation}
    \frac{\sum_{s=1}^{t} [ Y_s - f(\theta_* - X_s)]}{t} \to 0, \quad \text{almost surely},
\end{equation}
by the martingale strong law of large number.
Note that, according to our definition of MME (see Section \ref{section:algorithm}), we have either $\hat{\theta}_t \in \mathbb{R}$, $\hat{\theta}_t = \infty$, or $\hat{\theta}_t = - \infty$. When $\hat{\theta}_t \in \mathbb{R}$, we have $\sum_{s=1}^{t} Y_s = \sum_{s=1}^{t} f(\hat{\theta}_t - X_s)$; when $\hat{\theta}_t = \infty$, we have $\sum_{s=1}^{t} Y_s \geq \sum_{s=1}^{t} f(\hat{\theta}_t - X_s) \geq \sum_{s=1}^{t} f(\theta_* - X_s)$; and when $\hat{\theta}_t = - \infty$, we have $\sum_{s=1}^{t} Y_s \leq \sum_{s=1}^{t} f(\hat{\theta}_t - X_s) \leq \sum_{s=1}^{t} f(\theta_* - X_s)$. In any case, we have
\begin{equation}
	\left| \frac{\sum_{s=1}^{t} Y_s - f(\theta_* - X_s)}{t} \right| \geq \left| \frac{\sum_{s=1}^{t} f( \hat{\theta}_t - X_s) - f(\theta_* - X_s)}{t} \right|.
\end{equation}
Also note that
\begin{equation}
    \left| \sum_{s=1}^{t} f( \hat{\theta}_t - X_s) - f(\theta_* - X_s) \right| =  \sum_{s=1}^{t} \left| f( \hat{\theta}_t - X_s) - f(\theta_* - X_s) \right|,
\end{equation}
due to the monotonicity of $f(\cdot)$.
Combining these results gives
\begin{equation}
    \frac{\sum_{s=1}^{t} \left| f( \hat{\theta}_t - X_s) - f(\theta_* - X_s) \right| }{t} \to 0 \quad \text{almost surely}.
\end{equation}

Define $g(\Delta) := \inf_{x \in \mathcal{X}} |f(\theta_* + \Delta - x) - f(\theta_*-x)|$ and $\Delta_t := \hat{\theta}_t - \theta_*$.
Since $g(\Delta_t) = \inf_{x \in \mathcal{X}} |f(\hat{\theta}_t - x) - f(\theta_*-x)| \leq | f( \hat{\theta}_t - X_s) - f(\theta_* - X_s) |$ for all $s$, we have $g(\Delta_t) \leq \frac{1}{t} \sum_{s=1}^{t} | f( \hat{\theta}_t - X_s) - f(\theta_* - X_s) |$, and therefore, $g(\Delta_t) \to 0$ almost surely.
Note that $g(\Delta)$ is strictly increasing on $[0,\infty)$, strictly decreasing on $(-\infty,0]$ and $g(0)=0$, due to the compactness of $\mathcal{X}$ and the strict monotonicity of $f(\cdot)$. Therefore, $\Delta_t \to 0$ almost surely, which implies that $\hat{\theta}_t \to \theta_*$ almost surely.

\end{proof}

\begin{lemma}\label{lemma:cgf}
	Consider $Y \sim Bernoulli(p)$. Let $\psi(\lambda; p)\coloneqq \log \mathbb{E} \left[ \exp(\lambda(Y-p))\right] = \log(1-p+pe^\lambda) - \lambda p$. Then, we have
	\begin{enumerate}
		\item $\mathbb{E}\left[\exp(\lambda(Y-p) - \psi(\lambda; p)) \right] = 1$.
		\item $\psi(\lambda; p) \leq \frac{1}{8}\lambda^2$ for any $\lambda \in \mathbb{R}$.
		\item $\psi(\lambda;p) \leq \left( e^\lambda - \lambda - 1 \right) p(1-p)$ and $\psi(-\lambda;p) \leq \left( e^\lambda - \lambda - 1 \right) p(1-p)$ for any $\lambda >0$.
	\end{enumerate}
\end{lemma}

\begin{proof}
By the definition of $\psi$, we get
\begin{align}
    \mathbb{E} \left[ \exp \left( \lambda(Y-p) - \psi(\lambda; p) \right) \right] = \frac{\mathbb{E} \left[ \exp \left( \lambda(Y-p) \right) \right]}{ \exp \left( \psi(\lambda; p) \right)}=1.
\end{align}
By Hoeffding's lemma, for any $\lambda \in \mathbb{R}$
\begin{equation}
    \mathbb{E} \left[ \exp(\lambda( Y - \mathbb{E}[Y]) ) \right] \leq \exp \left( \frac{\lambda^2}{8} \right),
\end{equation}
which implies the second result.

Define a function
\begin{align}
	\varphi(\lambda) \coloneqq\; & \psi(\lambda; p) -  \left( e^\lambda - \lambda - 1 \right) p(1-p)\\
	=\; &\ln(1 - p +pe^\lambda) - \lambda p - \left( e^\lambda - \lambda - 1 \right) p(1-p).
\end{align}
Note that $\varphi(0) = 0$, then we now claim that $\varphi$ decreases in $\lambda>0$, i.e., $\varphi'(\lambda) < 0$ for $\lambda > 0$.
We have
\begin{align}
	\varphi '(\lambda) =\; & \frac{p e^\lambda}{1 - p + p e^\lambda} - p - \left(e^\lambda - 1 \right) p(1-p)
\end{align}
and
\begin{align}
	\varphi ''(\lambda) =\; & \frac{p e^\lambda}{1 - p + p e^\lambda} - \frac{\big( p e^\lambda \big)^2}{\big(1 - p + p e^\lambda \big)^2} - p(1-p) e^\lambda \\
	=\; & \frac{p (1-p) e^\lambda}{\big(1 - p + p e^\lambda \big)^2} - p(1-p) e^\lambda \\
    =\; & p(1-p)e^\lambda \left[ \frac{1}{(1-p+pe^\lambda)^2} - 1 \right].
\end{align}
Note that $\varphi'(0) = 0$ and $\varphi''(\lambda)<0$ since $1-p + pe^\lambda >0$ for $\lambda >0$.
Therefore, we can conclude that $\varphi'(\lambda)<0$ for $\lambda >0$.
\end{proof}

\begin{lemma}[Ville's inequality]\label{lemma:Ville}
    Let $X_0, X_1, \cdots$ be a nonnegative supermartingale. Then, for any real number $a > 0$,
    \begin{equation}
        \mathbb{P} \left( \sup_{t > 0} X_t \geq a \right) \leq \frac{\mathbb{E}[X_0]}{a},
    \end{equation}
    or equivalently, 
    \begin{equation}
        \mathbb{P} \Big( \exists t \geq 1 : X_t \geq a \Big) \leq \frac{\mathbb{E}[X_0 ]}{a}.
    \end{equation}
\end{lemma}

\newpage

\section{Proofs of Results in Section \ref{subsec:FB-analysis} (Fixed-Budget Setting)} \label{appendix:proof-FB}

\subsection{Proof of Theorem \ref{thm:FB-bound}}

\begin{proof}
Let $\mathcal{E}_T$ be the event that $\hat{\theta}_T$ falls within the interval $( \theta_* - \epsilon, \theta_* + \epsilon )$, i.e. $\mathcal{E}_T \coloneqq \big\{ \hat{\theta}_T \in ( \theta_* - \epsilon, \theta_* + \epsilon ) \big\}$, 
then the failure probability can be written as $p^\text{FB}(\pi; \theta_*, \epsilon, T) = 1 - \mathbb{P}_{\theta_*}(\mathcal{E}_T)$.

Fix $\theta$ satisfying $| \theta - \theta_*| > \epsilon$.
As $T \to \infty$, by Proposition~\ref{prop:MME-consistency}, we have $\mathbb{P}_{\theta_*}(\mathcal{E}_T) \to 1$, and $\mathbb{P}_{\theta} (\mathcal{E}_T) \to 0$.
Also fix $\eta \in (0,1)$. Then, there exists $t_0$ such that, for all $T > t_0$, $\mathbb{P}_{\theta}(\mathcal{E}_T) \leq \eta < \mathbb{P}_{\theta_*} ( \mathcal{E}_T )$, and thus,
\begin{align}
	d \big( \mathbb{P}_{\theta} ( \mathcal{E}_T) \mid \mathbb{P}_{\theta_*} (\mathcal{E}_T) \big) \geq\;  & d \big( \eta \mid \mathbb{P}_{\theta_*} (\mathcal{E}_T) \big)\\
	=\; & d \big( \eta \mid 1 - p^\text{FB}(\pi; \theta_*, \epsilon, T) \big)\\
	=\; & \eta \log \frac{\eta}{1 - p^\text{FB}(\pi; \theta_*, \epsilon, T)} + (1-\eta) \log \frac{1-\eta}{p^\text{FB}(\pi; \theta_*, \epsilon, T)}\\
	\geq\; & \eta \log \eta + (1-\eta) \log \frac{1-\eta}{p^\text{FB}(\pi; \theta_*, \epsilon, T)}.
\end{align}

On the other hand, by the lemma 2 of \cite{bassamboo2023learning}, we have
\begin{equation}
	\int_{\mathcal{X}} d \big( f(\theta - x) \mid f(\theta_* - x) \big) d\mu_{\theta}^{\pi} \geq d \big( \mathbb{P}_{\theta} ( \mathcal{E}_T) \mid \mathbb{P}_{\theta_*} (\mathcal{E}_T) \big).
\end{equation}
where $\mu_{\theta}^{\pi}$ is a probability measure on the query set $\mathcal{X}$ defined as $\mu_\theta^\pi ( (-\infty, x]) = \mathbb{E}_{\theta}^{\pi_{\epsilon,T}} \left[ \sum_{t=1}^{T} \mathbb{P} ( X_t \leq x \mid \mathcal{F}_{t-1}) \right]$.
Since $\mu_{\theta}^\pi(\mathcal{X}) = T$, we have
\begin{equation}
    \int_{\mathcal{X}} d \big( f(\theta - x) \mid f(\theta_* - x) \big) d\mu_{\theta}^{\pi} \leq T \cdot \sup_{x \in \mathcal{X}} d \big( f(\theta - x) \mid f(\theta_* - x) \big).
\end{equation}
Combining these results, we get that
\begin{equation}
    \eta \log \eta + (1-\eta) \log \frac{1-\eta}{p^\text{FB}(\pi; \theta_*, \epsilon, T)}
    \leq
    T \cdot \sup_{x \in \mathcal{X}} d \big( f(\theta - x) \mid f(\theta_* - x) \big),
\end{equation}
for $T > t_0$, and therefore,
\begin{equation}
    \limsup_{T \to \infty} \frac{1 - \eta}{T} \log \frac{1-\eta}{ p^\text{FB}(\pi; \theta_*, \epsilon, T)} \leq \sup_{x \in \mathcal{X}} d \big( f(\theta - x) \mid f(\theta_* - x) \big).
\end{equation}

Since the choice of $\eta$ was arbitrary, by taking $\eta \searrow 0$, we deduce that
\begin{equation}
	\limsup_{T \to \infty} - \frac{1}{T} \log p^\text{FB}(\pi; \theta_*, \epsilon, T) \leq \sup_{x \in \mathcal{X}} d \big( f(\theta - x) \mid f(\theta_* - x) \big).
\end{equation}
Since this inequality holds for any $\theta$ such that $|\theta - \theta_*| > \epsilon$,
\begin{equation}
    \limsup_{T \to \infty} - \frac{1}{T} \log p^\text{FB}(\pi; \theta_*, \epsilon, T) \leq \inf_{\theta : | \theta - \theta_* | > \epsilon} \sup_{x \in \mathcal{X}} d \big( f(\theta - x) \mid f(\theta_* - x) \big).
\end{equation}

We now establish an upper bound on the right hand side using Fisher information.
Lemma~\ref{lemma:KL-approximation} shows that there exists a constant $C_1 >0$ (possibly $\theta_*$-dependent) such that
\begin{equation}
    \sup_{x \in \mathcal{X}}\left| d \big( f(\theta_* - x + \epsilon) \mid f(\theta_* - x) \big) - \frac{I(x;\theta_*)}{2}\epsilon^2 \right| \leq C_1 \epsilon^3,
\end{equation}
which implies that
\begin{equation}
    \sup_{x \in \mathcal{X}} d \big( f(\theta_* - x + \epsilon) \mid f(\theta_* - x) \big)
        \; \leq \;
        \sup_{x \in \mathcal{X}} \frac{I(x;\theta_*)}{2}\epsilon^2 + C_1 \epsilon^3
        \; = \;
        \frac{I(x_*;\theta_*)}{2}\epsilon^2 + C_1 \epsilon^3.
\end{equation}
Similarly, there exists a constant $C_2 >0$ such that
\begin{equation}
   \sup_{x \in \mathcal{X}} d \big( f(\theta_* - x - \epsilon) \mid f(\theta_* - x) \big)
   \; \leq \; \frac{I(x_*;\theta_*)}{2}\epsilon^2 + C_2 \epsilon^3.
\end{equation}
Therefore,
\begin{align}
    \inf_{\theta: |\theta - \theta_*| >\epsilon}\sup_{x \in \mathcal{X}} d \big(f(\theta - x) \mid f(\theta_* - x) \big)
    &\leq \min_{\theta \in \{\theta_* + \epsilon, \theta_* - \epsilon\}} \left\{ \sup_{x \in\mathcal{X}} d \big( f(\theta - x) \mid f(\theta_* - x) \big) \right\}
    \\&\leq \min \left\{ \frac{I(x_*;\theta_*)}{2}\epsilon^2 + C_1 \epsilon^3, \frac{I(x_*;\theta_*)}{2}\epsilon^2 + C_2 \epsilon^3 \right\}
    \\&= \frac{I(x_*;\theta_*)}{2}\epsilon^2 + \min\{ C_1, C_2 \} \epsilon^3,
\end{align}
which concludes the proof.

\end{proof}

\subsection{Proof of Theorem \ref{thm:FIT-Q-FB}}

\begin{proof}
To obtain the desired result, it is sufficient to show that
\begin{equation}
	\limsup_{\epsilon \to 0} \limsup_{T \to \infty} \frac{1}{\epsilon^2 T} \log \mathbb{P} \left( | \hat{\theta}_T - \theta_* | \geq \epsilon \right) \leq - \frac{I(x_*;\theta_*)}{2}.
\end{equation}
Denote $g_* \coloneqq f' (\theta_* - x_*)$, and $v_* \coloneqq f(\theta_* - x_*) \big( 1 - f(\theta_* - x_*) \big)$, then the maximum Fisher information $I(x_* ; \theta_*)$ can be expressed as 
\begin{equation}
    I(x_*; \theta_*) = \frac{f'(\theta_* - x_*)^2}{ f(\theta_* - x_*) \big(1 - f(\theta_* - x_*) \big)} = \frac{g_*^2}{v_*}.
\end{equation}

\paragraph{Step 1. Failure probability decomposition.}
Define the good event $\mathcal{E}_0$ as
\begin{equation}
	\mathcal{E}_0 \coloneqq \left\{ |\hat{\theta}_t - \theta_* | < \sqrt{\epsilon}, \;\;\;\; \forall t \geq t_0 \right\}.
\end{equation}
We can decompose the failure probability as
\begin{align}
	\mathbb{P} \left( | \hat{\theta}_T - \theta_* | \geq \epsilon \right) =\; & \mathbb{P} \left( | \hat{\theta}_T - \theta_* | \geq \epsilon , \;\; \mathcal{E}_0 \right) + \mathbb{P} \left( | \hat{\theta}_T - \theta_* | \geq \epsilon , \;\; \mathcal{E}_0^c \right)\\
	\leq\; & \mathbb{P} \left( | \hat{\theta}_T - \theta_* | \geq \epsilon , \;\; \mathcal{E}_0 \right) + \mathbb{P} \left( \mathcal{E}_0^c \right)\\
	\leq\; & \underbrace{ \mathbb{P} \left( \hat{\theta}_T \geq \theta_* + \epsilon, \;\; \mathcal{E}_0 \right) }_{\text{(A)}} + \underbrace{ \mathbb{P} \left( \hat{\theta}_T \leq \theta_* - \epsilon, \;\; \mathcal{E}_0 \right) }_{\text{(B)}} + \underbrace{ \mathbb{P} \left( \mathcal{E}_0^c \right) }_{\text{(C)}}.
\end{align}
We bound each term after introducing useful notation.

Define the following constants:
\begin{equation}
	\kappa(\epsilon) \coloneqq  \inf_{x \in \mathcal{X}} \min_{ \theta \in \{ \theta_* + \sqrt{\epsilon}, \theta_* - \sqrt{\epsilon}\}} |f(\theta - x) - f(\theta_* - x) |,
\end{equation}
\begin{equation}
	g_0(\epsilon) \coloneqq \inf_{x : | x - x_*| \leq \sqrt{\epsilon}} \min_{\theta \in \{ \theta_* + \epsilon, \theta_* - \epsilon \} } \frac{ \big| f(\theta' - x ) - f(\theta_* - x) \big|}{\epsilon},
\end{equation}
and
\begin{equation}
	v_0(\epsilon) \coloneqq \sup_{x: |x - x_*| \leq \sqrt{\epsilon}} f(\theta_* - x) \big( 1- f(\theta_* - x) \big).
\end{equation}
Further define a threshold $t_0(\epsilon,T)$:
\begin{equation}
	t_0(\epsilon, T) \coloneqq \left\lceil \frac{g_*^2}{4v_* \kappa(\epsilon)^2} \times \epsilon^2 T \right\rceil.
\end{equation}
Note that we have
\begin{equation}
    \lim_{\epsilon \to 0} g_0(\epsilon) = g_*
    , \quad
    \lim_{\epsilon \to 0} v_0(\epsilon) = v_*
    , \quad
    \text{and} \quad
    \lim_{\epsilon \to 0} \lim_{T \to \infty} \frac{t_0(\epsilon, T)}{T} = 0.
\end{equation}

\paragraph{Step 2. Bounding the term (A).}
Note that for any $\lambda > 0$,
\begin{align}
    &\left\{ \hat{\theta}_T \geq \theta_* + \epsilon \right\}
    \\&\stackrel{(a)}{=} \left\{ \sum_{t=1}^{T} f(\hat{\theta}_T - X_t) \geq \sum_{t=1}^{T} f(\theta_* - X_t + \epsilon) \right\}
    \\&\stackrel{(b)}{\subseteq} \left\{ \sum_{t=1}^{T} Y_t  \geq \sum_{t=1}^{T} f(\theta_* - X_t + \epsilon)  \right\}
    \\&= \left\{ \sum_{t=1}^{T} \big( Y_t - f(\theta_* - X_t)\big) \geq \sum_{t=1}^{T} \big( f(\theta_* - X_t + \epsilon) - f(\theta_* - X_t) \big) \right\}
    \\&= \left\{ \lambda \cdot \sum_{t=1}^{T} \big( Y_t - f(\theta_* - X_t)\big) \geq \lambda \cdot \sum_{t=1}^{T} \big( f(\theta_* - X_t + \epsilon) - f(\theta_* - X_t) \big) \right\}
    \\&= \left\{ \sum_{t=1}^{T} \Big[ \lambda\big( Y_t - f(\theta_* - X_t)\big) - \psi(\lambda, f(\theta_* - X_t)) \Big] \geq \sum_{t=1}^{T} \Big[ \lambda\big( f(\theta_* - X_t + \epsilon) - f(\theta_* - X_t) \big) - \psi(\lambda, f(\theta_* - X_t)) \Big] \right\},
\end{align}
where step (a) uses the monotonicity of $f$, and step (b) uses the fact that, if $\hat{\theta}_T \geq \theta_* + \epsilon$, then $\hat{\theta}_T = \infty$ or $\hat{\theta}_T \in \mathbb{R}$, which implies that
\begin{equation}
    \sum_{t=1}^{T} Y_t \geq \sum_{t=1}^{T} f(\hat{\theta}_T - X_t).
\end{equation}

By Lemma \ref{lemma:cgf}, we have
\begin{equation}
	\psi(\lambda; f(\theta_* - X_t) ) \leq \left( e^\lambda - \lambda - 1\right) f(\theta_* - X_t) \big( 1 - f(\theta_* - X_t) \big).
\end{equation}
Under $\mathcal{E}_0$, for any $t \geq t_0 +1$, since $|X_t - x_*| < \sqrt{\epsilon}$, by the definition of constants $g_0$ and $v_0$, we get 
\begin{equation}
	\lambda \big( f(\theta_* - X_t + \epsilon) - f(\theta_* - X_t) \big) - \psi(\lambda; f(\theta_* - X_s)) \geq \lambda g_0 \epsilon - \left(e^\lambda - \lambda - 1 \right) v_0.
\end{equation}
On the other hand, $\mathcal{E}_0$ does not contain any implication about the behavior of $X_t$ for $t \leq t_0$, but we can exploit the universal facts that $f(\theta_* - X_t + \epsilon) - f(\theta_* - X_t) \geq 0$ and $\psi(\lambda; f(\theta_* - X_t) ) \leq \lambda^2/8$.
Later, it is shown that the effect of these loose bounds becomes negligible in the asymptotic regime.
Combining them, we finally get 
\begin{align}
	\mathbb{P} \Big(& \hat{\theta}_T \geq \theta_* + \epsilon, \;\; \mathcal{E}_0 \Big)\\
	 \leq\; & \mathbb{P} \left( \sum_{t=1}^{T} \Big[ \lambda \big( Y_t - f(\theta_* - X_t) \big) - \psi(\lambda; f(\theta_* - X_t)) \Big] \geq \sum_{t=1}^{T} \Big[ \lambda \big(f(\theta_* - X_t + \epsilon) - f(\theta_* - X_t) \big) - \psi(\lambda; f(\theta_* - X_t)) \Big], \;\; \mathcal{E}_0 \right)\\
	\leq\; & \mathbb{P} \left( \sum_{t=1}^{T} \Big[ \lambda \big( Y_t - f(\theta_* - X_t) \big) - \psi(\lambda; f(\theta_* - X_t)) \Big] \geq \sum_{t=1}^{t_0}\left[ 0 - \frac{1}{8} \lambda^2 \right] + \sum_{t=t_0 + 1}^{T} \left( \lambda g_0 \epsilon - \left(e^\lambda - \lambda - 1\right) v_0 \right) \right)\\
	=\; & \mathbb{P} \left( \exp\left( \sum_{t=1}^{T} \Big[ \lambda \big( Y_t - f(\theta_* - X_t) \big) - \psi(\lambda; f(\theta_* - X_t)) \Big] \right) \geq \exp \left( -\frac{1}{8}\lambda^2 t_0 + (T - t_0)  \left( \lambda g_0 \epsilon - \left( e^\lambda - \lambda - 1 \right) v_0 \right) \right) \right)\\
	\leq\; & \frac{1}{\exp \Big( -\frac{1}{8}\lambda^2 t_0 + (T - t_0)  \left( \lambda g_0 \epsilon - \left(e^\lambda - \lambda - 1\right) v_0 \right) \Big)}.
\end{align}
The last inequality is due to Ville's inequality (Lemma \ref{lemma:Ville}).

\paragraph{Step 3. Bounding the term (B).}
Similarly,
\begin{align}
	\mathbb{P} \Big(& \hat{\theta}_T \leq \theta_* - \epsilon, \;\; \mathcal{E}_0 \Big)\\
	 =\; & \mathbb{P} \left( \sum_{t=1}^{T} \Big[ \lambda \big( f(\theta_* - X_t) - Y_t \big) - \psi(-\lambda; f(\theta_* - X_t)) \Big] \geq \sum_{t=1}^{T} \Big[ \lambda \big(f(\theta_* - X_t + \epsilon) - f(\theta_* - X_t) \big) - \psi(-\lambda; f(\theta_* - X_t)) \Big], \;\; \mathcal{E}_0 \right)\\
	\leq\; & \mathbb{P} \left( \sum_{t=1}^{T} \Big[ \lambda \big( f(\theta_* - X_t) - Y_t \big) - \psi(-\lambda; f(\theta_* - X_t)) \Big] \geq \sum_{t=1}^{t_0}\left[ 0 - \frac{1}{8} \lambda^2 \right] + \sum_{t=t_0 + 1}^{T} \left( \lambda g_0 \epsilon - \left( e^\lambda - \lambda - 1\right) v_0 \right) \right)\\
	=\; & \mathbb{P} \left( \exp\left( \sum_{t=1}^{T} \Big[ \lambda \big( f(\theta_* - X_t) - Y_t \big) - \psi(-\lambda; f(\theta_* - X_t)) \Big] \right) \geq \exp \left( -\frac{1}{8}\lambda^2 t_0 + (T - t_0)  \left( \lambda g_0 \epsilon - \left( e^\lambda - \lambda - 1\right)v_0 \right) \right) \right)\\
	\leq\; & \frac{1}{\exp \left( -\frac{1}{8}\lambda^2 t_0 + (T - t_0)  \left( \lambda g_0 \epsilon - \left( e^\lambda - \lambda - 1\right) v_0 \right) \right)}.
\end{align}

\paragraph{Step 4. Bounding the term (C).}
Following from Proposition \ref{prop:maximal}, we have
\begin{align}
	\mathbb{P} \left( \mathcal{E}_0^c \right) =\; & \mathbb{P} \left( \exists t \geq t_0 : |\hat{\theta}_t - \theta_* | \geq \sqrt{\epsilon} \right)\\
	\leq\; & 2\exp \left( - 2 \kappa^2 t_0 \right)\\
	=\; & 2\exp \left( - 2 \kappa^2 \left\lceil \frac{g_*^2}{4v_* \kappa^2} \times \epsilon^2 T \right\rceil\right)\\
	\leq\; & 2\exp \left( - \frac{g_*^2}{2v_* } \times \epsilon^2 T \right).
\end{align}

\paragraph{Step 5. Computing the limit.}
Combining above results, we deduce that
\begin{align}
	\mathbb{P} \left( | \hat{\theta}_T - \theta_* | \geq \epsilon \right) =\; & \mathbb{P} \left( | \hat{\theta}_T - \theta_* | \geq \epsilon, \;\; \mathcal{E}_0 \right) + \mathbb{P} \left( \mathcal{E}_0^c \right) \\
	\leq\; & 2 \exp \left(-\gamma(\epsilon, T) \right) + 2\exp \left( - \frac{g_*^2}{2v_* } \times \epsilon^2 T \right),
\end{align}
where $\gamma(\epsilon, T) \coloneqq -\frac{1}{8}\lambda^2 t_0 + (T - t_0)  \left( \lambda g_0 \epsilon - \left( e^\lambda - \lambda - 1\right) v_0 \right)$.
Here, we set $\lambda \coloneqq g_0 \epsilon / v_0$, then we have $\lim_{\epsilon \to 0} \lambda / \epsilon = g_* /v_*$, and thus
\begin{align}
    \limsup_{\epsilon \to 0}\limsup_{T \to \infty} \frac{-\gamma(\epsilon,T)}{\epsilon^2 T} =\; & \limsup_{\epsilon \to 0}\limsup_{T \to \infty} \frac{\frac{1}{8}\lambda^2 t_0 - (T - t_0)  \left( \lambda g_0 \epsilon - \left( e^\lambda - \lambda - 1\right) v_0 \right)}{\epsilon^2 T}\\
    =\; & \limsup_{\epsilon \to 0} \frac{-\lambda g_0 \epsilon + \left(e^\lambda - \lambda - 1\right) v_0}{\epsilon^2}\\
    =\; & \limsup_{\epsilon \to 0} \left[ -\frac{\lambda g_0}{\epsilon} + \frac{e^\lambda - \lambda- 1}{\epsilon^2} v_0\right]\\
    =\; & -\frac{g_*}{v_*} g_* + \frac{1}{2} \left(\frac{g_*}{v_*}\right)^2 v_*\\
    =\; & -\frac{g_*^2}{2v_*}.
\end{align}
We also have
\begin{equation}
    \limsup_{\epsilon\to0}\limsup_{T \to \infty} \frac{1}{\epsilon^2 T} \log \exp\left( - \frac{g_*^2}{2v_*} \times \epsilon^2 T \right) = - \frac{g_*^2}{2v_*}.
\end{equation}

Finally, since
	\begin{equation}
		\limsup_{n \to \infty} \frac{1}{n}\log \big( a_n + b_n \big) \leq \max \left\{ \limsup_{n \to \infty} \frac{1}{n} \log a_n, \limsup_{n \to \infty} \frac{1}{n} \log b_n \right\},
	\end{equation}
for any real-valued sequences
$(a_n)_{n\in\mathbb{N}}$ and $(b_n)_{n\in\mathbb{N}}$, we conclude that
\begin{equation}
    \limsup_{\epsilon \to 0}\limsup_{T \to \infty} \frac{1}{\epsilon^2 T} \log \mathbb{P} \left( | \hat{\theta}_T - \theta_* | \geq \epsilon \right) \leq - \frac{g_*^2}{2v_*} = -\frac{I(x_*; \theta_*)}{2},
\end{equation}
which is equivalent to
\begin{equation}
    \liminf_{T \to \infty} - \frac{1}{ T} \log \mathbb{P} \left( | \hat{\theta}_T - \theta_* | \geq \epsilon \right) \geq  \frac{I(x_*; \theta_*)}{2} \epsilon^2 + o(\epsilon^2).
\end{equation}
\end{proof}

\newpage
\section{Proofs of Results in Section \ref{subsec:FC-analysis} (Fixed-Confidence Setting)}\label{appendix:proof-FC}

\subsection{Proof of Theorem \ref{thm:FC-bound}}

\begin{proof}
Fix $\theta$ satisfying $| \theta - \theta_*| > \epsilon$, and pick $\eta \coloneqq | \theta - \theta_*| - \epsilon > 0$. Define an event $\mathcal{E}_\theta \coloneqq \big\{ \hat{\theta}_{\tau} \not\in (\theta - \epsilon, \theta + \epsilon) \big\}$.

For any FC-consistent algorithm $\pi$, we have  $\mathbb{P}_{\theta}^{\pi_{\epsilon,\delta}}(\mathcal{E}_\theta) \leq \delta$ by Condition \ref{cond:FC-defn-correctness}, and $\mathbb{P}_{\theta_*}^{\pi_{\epsilon,\delta}} ( | \theta_{\tau} - \theta_* | < \eta ) \to 1$ as $\delta \to 0$ by Condition \ref{cond:FC-defn-stability}.
Since $\{ | \theta_{\tau} - \theta_* | < \eta \} \subset \mathcal{E}_\theta$, we get that $\mathbb{P}_{\theta_*}^{\pi_{\epsilon,\delta}}(\mathcal{E}_\theta) \to 1$ as $\delta \to 0$.
Therefore, for sufficiently small $\delta$, we have $\mathbb{P}_{\theta_*}^{\pi_{\epsilon,\delta}} (\mathcal{E}_\theta) \geq \delta \geq \mathbb{P}_{\theta}^{\pi_{\epsilon,\delta}}(\mathcal{E}_\theta)$, and hence $d \big( \mathbb{P}_{\theta_*}^{\pi_{\epsilon,\delta}} (\mathcal{E}_\theta) \mid \mathbb{P}_{\theta}^{\pi_{\epsilon,\delta}}(\mathcal{E}_\theta) \big) \geq d \big( \mathbb{P}_{\theta_*}^{\pi_{\epsilon,\delta}} (\mathcal{E}_\theta) \mid \delta \big)$.
Therefore, with $p_\delta \coloneqq \mathbb{P}_{\theta_*}^{\pi_{\epsilon,\delta}} (\mathcal{E}_\theta)$,
\begin{align}
    \liminf_{\delta \to 0} \frac{d \big( \mathbb{P}_{\theta_*}^{\pi_{\epsilon,\delta}} (\mathcal{E}_\theta) \mid \mathbb{P}_{\theta}^{\pi_{\epsilon,\delta}}(\mathcal{E}_\theta) \big)}{\log(1/\delta)}
    &\geq \liminf_{\delta \to 0} \frac{d \big( p_\delta \mid \delta  \big)}{\log(1/\delta)}
    \\&= \liminf_{\delta \to 0} \frac{ p_\delta \log \frac{p_\delta}{\delta} + (1-p_\delta) \log \frac{1-p_\delta}{1-\delta}}{\log(1/\delta)}
    \\&= 1.
\end{align}

Let $\mu_{\theta_*}^{\pi_{\epsilon,\delta}}$ be the measure defined on the query set $\mathcal{X}$, satisfying $
    \mu_{\theta_*}^{\pi_{\epsilon,\delta}}( (-\infty, x] ) = \mathbb{E}_{\theta_*}^{\pi_{\epsilon,\delta}}\big[ \sum_{t=1}^\tau \mathbb{P}( X_t \leq x \mid \mathcal{F}_{t-1} ) \big]
$ for all $x \in \mathbb{R}$.
Since $\int_{\mathcal{X}} d\mu_{\theta_*}^{\pi_{\epsilon,\delta}} = \mathbb{E}_{\theta_*}^{\pi_{\epsilon,\delta}}[\tau]$ by the definition, we can define a distribution $F_{\theta_*}^{\pi_{\epsilon,\delta}}$ on $\mathcal{X}$ corresponding to the measure $\mu_{\theta_*}^{\pi_{\epsilon,\delta}}$ as $F_{\theta_*}^{\pi_{\epsilon,\delta}} = \mu_{\theta_*}^{\pi_{\epsilon,\delta}}/\mathbb{E}_{\theta_*}^{\pi_{\epsilon,\delta}}[\tau]$.
Then, we have
\begin{align}
    \mathbb{E}_{\theta_*}^{\pi_{\epsilon,\delta}} \left[ \tau \right]
    &=
    \frac{ \int_{\mathcal{X}} d \big( f(\theta_* - x) \mid f(\theta - x) \big) d\mu_{\theta_*}^{\pi_{\epsilon,\delta}} }{\int_{\mathcal{X}} d \big( f(\theta_* - x) \mid f(\theta - x) \big) dF_{\theta_*}^{\pi_{\epsilon,\delta}}}
    \\&\geq 
    \frac{ \int_{\mathcal{X}} d \big( f(\theta_* - x) \mid f(\theta - x) \big) d\mu_{\theta_*}^{\pi_{\epsilon,\delta}} }{ \sup_{x \in \mathcal{X}} d \big( f(\theta_* - x) \mid f(\theta - x) \big)}
    \\&\geq \frac{ d \big( \mathbb{P}_{\theta_*}^{\pi_{\epsilon,\delta}} ( \mathcal{E}_\theta) \mid \mathbb{P}_{\theta}^{\pi_{\epsilon,\delta}} (\mathcal{E}_\theta) \big) }{ \sup_{x \in \mathcal{X}} d \big( f(\theta_* - x) \mid f(\theta - x) \big)},
\end{align}
where the last inequality is due to lemma 2 of \citet{bassamboo2023learning}.
Combining this with the above result gives
\begin{equation}
    \liminf_{\delta \to 0} \frac{\mathbb{E}_{\theta_*}^{\pi_{\epsilon,\delta}} \left[ \tau \right]}{\log(1/\delta)}
    \geq 
    \liminf_{\delta \to 0}\frac{d \big(\mathbb{P}_{\theta_*}^{\pi_{\epsilon,\delta}} (\mathcal{E}_\theta) \mid \mathbb{P}_{\theta}^{\pi_{\epsilon,\delta}} (\mathcal{E}_\theta) \big)}{\log(1/\delta)} \left[ \sup_{x \in \mathcal{X}} d \big(f(\theta_* - x) \mid f(\theta - x) \big) \right]^{-1}
    \geq \left[ \sup_{x \in \mathcal{X}} d \big(f(\theta_* - x) \mid f(\theta - x) \big) \right]^{-1}.
\end{equation}
Since this inequality holds for any $\theta$ such that $|\theta - \theta_*| > \epsilon$, we have
\begin{equation}
    \liminf_{\delta \to 0} \frac{\mathbb{E}_{\theta_*}^{\pi_{\epsilon,\delta}} \left[ \tau \right]}{\log(1/\delta)}
    \geq \left[ \inf_{\theta: |\theta - \theta_*| > \epsilon} \sup_{x \in \mathcal{X}} d \big(f(\theta_* - x) \mid f(\theta - x) \big) \right]^{-1}.
\end{equation}
Similarly to the proof of Theorem~\ref{thm:FB-bound}, we can show that $\inf_{\theta: |\theta - \theta_*| > \epsilon} \sup_{x \in \mathcal{X}} d \big(f(\theta_* - x) \mid f(\theta - x) \big) = \frac{I(x_*;\theta_*)}{2} \epsilon^2 + O(\epsilon^3)$, which concludes the proof.\footnote{The only minor difference is that we here have $d \big(f(\theta_* - x) \mid f(\theta - x) \big)$ instead of $d \big(f(\theta - x) \mid f(\theta_* - x) \big)$, but Lemma~\ref{lemma:KL-approximation} also covers this case.}
\end{proof}

\newpage

\subsection{Proof of Theorem \ref{thm:FIT-Q-FC-consistency}}

\begin{lemma}\label{lemma:ts-convergence}
    Suppose that the function $\phi( \lambda_*, \cdot)$ satisfies Condition~\ref{cond:FC-cons-quasi-convex} of Theorem~\ref{thm:FIT-Q-FC-consistency}.
    Further assume that $\hat{\theta}_t \to \theta_*$ and $X_t \to x_*$ almost surely as $t \to \infty$.
    Then, the test statistics satisfies
    \begin{equation}
        \lim_{t \to \infty} \frac{Z_t^\epsilon}{t} = \min_{z \in \{z_*-\epsilon, z_* + \epsilon\}} \big[\lambda_* |f(z) - f(z_*)| - \phi(\lambda_*, f(z))\big]
        , \quad \text{almost surely.}
    \end{equation}
\end{lemma}

\begin{proof}

Given condition \ref{cond:FC-cons-quasi-convex}, $\phi(\lambda_*, \cdot)$ is continuous, then a function defined as
\begin{equation}
    G_1(\theta, x) \coloneqq \lambda_* \big| f(\theta - x + \epsilon) - f(\theta - x) \big| - \phi(\lambda_*, f(\theta - x + \epsilon))
\end{equation}
is also continuous in $\theta$ and $x$, and bounded, i.e., there exists $M$ such that $|G_1| \leq M$.
Due to the continuity of $G_1$, for any $\xi >0$, there is $\eta$ such that if $|\theta - \theta_*| < \eta$ and $| x - x_* | < \eta$, then
\begin{equation}
    \big| G_1(\theta, x) -  G_1(\theta_*, x_* )\big| < \xi.
\end{equation}
For any fixed sample path, there exist $t_1(\eta)$ such that $|\hat{\theta}_t - \theta_* | < \eta $ for $t \geq t_1(\eta)$ and $t_2(\eta)$ such that $|X_t - x_* |< \eta$ for $t \geq t_2(\eta)$ for any $\eta<0$.
For any $t \geq t_0 \coloneqq \max\{ t_1(\eta), t_2(\eta) \}$, we have $|\hat{\theta}_t - \theta_*| <\eta$  and $|X_t - x_* | < \eta$, and thus
\begin{align}
    \left| \frac{1}{t} \sum_{s=1}^{t} G_1(\hat{\theta}_t, X_s) - G_1(\theta_*, x_*) \right| \leq \frac{t_0 - 1}{t} M + \frac{1}{t}\sum_{s =t_0}^{t} \xi.
\end{align}
Letting $t \to \infty$ makes the right hand side converge to $\xi$.
Since $\xi$ is arbitrary, we deduce that
\begin{equation}
    \lim_{t \to \infty} \frac{1}{t} \sum_{s=1}^{t} G_1(\hat{\theta}_t, X_s)
    = G_1(\theta_*, x_* ),
\end{equation}
almost surely.
Similarly, a function defined as 
\begin{equation}
    G_2(\theta, x) \coloneqq \lambda_*  \big| f(\theta - x - \epsilon) - f(\theta - x) \big| - \phi(\lambda_*, f(\theta - x -\epsilon))
\end{equation}
has the same convergence result, i.e., given that $\hat{\theta}_t \to \theta_*$ and $X_t \to x_*$ almost surely,
\begin{equation}
    \lim_{t \to \infty} \frac{1}{t} \sum_{s=1}^{t} G_1(\hat{\theta}_t, X_s)
    = G_1(\theta_*, x_* ),
\end{equation}
almost surely.
Finally,
\begin{align}
    \lim_{t \to \infty} \frac{Z_t^\epsilon}{t}
    &= \lim_{t \to \infty} \frac{1}{t} \min_{\theta \in \{ \hat{\theta}_t + \epsilon, \hat{\theta}_t - \epsilon \}}  \sum_{s=1}^{t} \left[ \lambda_* \left| f(\hat{\theta}_t - X_s) - f(\theta - X_s) \right| - \phi(\lambda_* , f(\theta - X_s) ) \right]
    \\&= \lim_{t \to \infty} \frac{1}{t} \min \left\{ \sum_{s=1}^{t} G_1(\hat{\theta}_t , X_s), \sum_{s=1}^{t} G_2(\hat{\theta}_t , X_s) \right\}
    \\&= \min\left\{ \lim_{t \to \infty} \frac{1}{t} \sum_{s=1}^{t} G_1(\hat{\theta}_t, X_s), \; \lim_{t \to \infty} \frac{1}{t} \sum_{s=1}^{t} G_2(\hat{\theta}_t, X_s ) \right\}
    \\& = \min\left\{ G_1(\theta_*, x_*), \; G_2(\theta_*, x_*) \right\}
    \\&= \min_{\theta \in \{ \theta_* + \epsilon, \theta_* - \epsilon \}}  \left[ \lambda_* \left| f(\theta_* - x_*) - f(\theta - x_*) \right| - \phi(\lambda_* , f(\theta - x_*) ) \right]
    \\&= \min_{z \in \{z_* + \epsilon, z_* - \epsilon\} }\left[ \lambda_* \left| f(z_*) - f(z) \right| - \phi(\lambda_* , f(z) ) \right].
\end{align}

\end{proof}

\begin{proof}[\textbf{Proof of Theorem \ref{thm:FIT-Q-FC-consistency}}]

We want to show that \texttt{FIT-Q} algorithm, denoted by $\pi_{\epsilon,\delta}^\texttt{FIT-Q}$, satisfies that for any $\theta_* \in \mathbb{R}$, $\epsilon \in \mathbb{R}^+$ and $\delta \in (0,1)$,
\begin{enumerate}[label=(\alph*)]
    \item Its stopping time $\tau$ is finite almost surely.
    \item Its failure probability does not exceed $\delta$.
    \begin{equation}
        \mathbb{P}_{\theta_*}^{\pi_{\epsilon,\delta}} \left( | \hat{\theta}_\tau - \theta_* | > \epsilon \right) \leq \delta.
    \end{equation}
    \item For any $\eta \in \mathbb{R}_+$,
    \begin{equation}
        \lim_{\delta \to 0} \mathbb{P}_{\theta_*}^{\pi_{\epsilon, \delta}^\texttt{FIT-Q}} \left( | \hat{\theta}_\tau - \theta_* | > \eta \right) = 0.
    \end{equation}
\end{enumerate}
We begin the proof by fixing $\theta_*$, $\epsilon$ and $\delta$.

\paragraph{Step 1. Showing that $\tau < \infty$ almost surely.}
Define
\begin{equation}
    c_* \coloneqq \min_{z \in \{z_*-\epsilon, z_* + \epsilon\}} \big[\lambda_* |f(z) - f(z_*)| - \phi(\lambda_*, f(z))\big].
\end{equation}
Condition~\ref{cond:FC-cons-positive} implies that $c_* > 0$, and Lemma \ref{lemma:ts-convergence} implies that $Z_t^\epsilon/t$ converges to $c_*$ almost surely.
If $\tau = \infty$, then $Z_t^\epsilon < \log\frac{2}{\delta}$ for all $t$, so that $Z_t^\epsilon/t$ should converge to $0$, which is a contradiction.
Therefore, $\tau < \infty$.

\paragraph{Step 2. Showing that $\mathbb{P}_{\theta_*}^{\pi_{\epsilon,\delta}} \left( | \hat{\theta}_\tau - \theta_* | > \epsilon \right) \leq \delta$.}
We begin by decomposing the failure probability:
\begin{align}
    \mathbb{P}\left(  | \hat{\theta}_\tau - \theta_* | > \epsilon \right)
    &\leq \mathbb{P}\left( \exists t : \, Z_t^\epsilon \geq \log \frac{2}{\delta}, \, |\hat{\theta}_t - \theta_* | > \epsilon \right)
    \\&\leq \underbrace{ \mathbb{P}\left( \exists t : \, Z_t^\epsilon \geq \log \frac{2}{\delta}, \, \theta_* > \hat{\theta}_t + \epsilon \right) }_{\text{(A)}}
    + \underbrace{ \mathbb{P}\left( \exists t : \, Z_t^\epsilon \geq \log \frac{2}{\delta}, \, \theta_* < \hat{\theta}_t - \epsilon \right) }_{\text{(B)}}.
\end{align}
We bound each term.

Define $Z_t^\text{in}(\theta)$ as
\begin{equation}
    Z_t^\text{in}(\theta) := \sum_{s=1}^{t} \left[ \lambda_* \big| f(\hat{\theta}_t - X_s) - f(\theta - X_s) \big| - \phi(\lambda_*, f(\theta - X_s)) \right].
\end{equation}
By Condition~\ref{cond:FC-cons-quasi-convex}, each summand $\lambda_* \big| f(\hat{\theta}_t - X_s) - f(\theta - X_s) \big| - \phi(\lambda_*, f(\theta - X_s))$ as a function $\theta$ is quasi-convex and minimized at $\theta=\hat{\theta}_t$, and therefore, so is $Z_t^\text{in}(\theta)$.
That is, $Z_t^\text{in}(\theta)$ is non-increasing on $(-\infty, \hat{\theta}_t]$ and non-decreasing on $[\hat{\theta}_t, \infty)$.

\textbf{To bound the term (A)}, fix $t$ and consider an event $\{ \theta_* > \hat{\theta}_t + \epsilon 
 \}$.
On this event, we have
\begin{align}
    Z_t^\epsilon
    &\stackrel{\text{(a)}}{=} \min\left\{ Z_t^\text{in}(\hat{\theta}_t-\epsilon), Z_t^\text{in}(\hat{\theta}_t+\epsilon) \right\}
    \\&\leq Z_t^\text{in}(\hat{\theta}_t+\epsilon)
    \\&\stackrel{\text{(b)}}{\leq} Z_t^\text{in}(\theta_*)
    \\&= \sum_{s=1}^{t} \left[ \lambda_* \big( f(\theta_* - X_s) - f(\hat{\theta}_t - X_s) \big) - \phi(\lambda_*, f(\theta_* - X_s)) \right]
    \\&\stackrel{\text{(c)}}{\leq}  \sum_{s=1}^{t}\left[ \lambda_* \big( f(\theta_* - X_s) - Y_s \big) - \phi(\lambda_*, f(\theta_* - X_s))\right],
\end{align}
where step (a) uses the definition of $Z_t^\epsilon$, and step (b) uses the fact that $Z_t^\text{in}(\cdot)$ is non-decreasing on $[\hat{\theta}_t, \infty)$ and $\theta_* \geq \hat{\theta}_t + \epsilon$ on this event, and step (c) uses the property of MME.
Define 
\begin{equation}
    M_t^+ := \exp\left( \sum_{s=1}^{t}\left[ \lambda_* \big( f(\theta_* - X_s) - Y_s \big) - \phi(\lambda_*, f(\theta_* - X_s))\right] \right).
\end{equation}
Due to Condition~\ref{cond:FC-cons-cgf-dominate}, $\{M_t^+\}_{t \in \mathbb{N}}$ is a nonnegative supermartingale such that $M_0^+ =1$.
Combining these results, we deduce that
\begin{align}
    \mathbb{P}\left( \exists t : \, Z_t^\epsilon \geq \log \frac{2}{\delta}, \, \theta_* > \hat{\theta}_t + \epsilon \right)
    &\leq \mathbb{P}\left( \exists t : \, \log M_t^+ \geq \log \frac{2}{\delta}, \, \theta_* > \hat{\theta}_t + \epsilon \right)
    \\&= \mathbb{P}\left( \exists t : \, M_t^+ \geq \frac{2}{\delta}, \, \theta_* > \hat{\theta}_t + \epsilon \right)
    \\&\leq \mathbb{P}\left( \exists t : \, M_t^+ \geq \frac{2}{\delta} \right)
    \\&\leq \frac{\delta}{2},
\end{align}
where the last inequality follows from Ville's inequality (Lemma~\ref{lemma:Ville}).

\textbf{To bound the term (B)}, consider an event $\{ \theta_* < \hat{\theta}_t - \epsilon \}$. We can show that, on this event,
\begin{align}
    Z_t^\epsilon
    &\leq Z_t^\text{in}(\hat{\theta}_t-\epsilon)
    \\&\leq Z_t^\text{in}(\theta_*)
    \\&= \sum_{s=1}^{t} \left[ \lambda_* \big( f(\hat{\theta}_t - X_s) - f(\theta_* - X_s) \big) - \phi(\lambda_*, f(\theta_* - X_s)) \right]
    \\&\leq  \sum_{s=1}^{t}\left[ \lambda_* \big( Y_s - f(\theta_* - X_s) \big) - \phi(\lambda_*, f(\theta_* - X_s))\right]
    \\&= \log M_t^-,
\end{align}
where
\begin{equation}
    M_t^- := \exp\left( \sum_{s=1}^{t}\left[ \lambda_* \big( Y_s - f(\theta_* - X_s) \big) - \phi(\lambda_*, f(\theta_* - X_s))\right] \right).
\end{equation}
Since $\{M_t^-\}_{t \in \mathbb{N}}$ is a nonnegative supermartingale such that $M_0^- = 1$, we have
\begin{align}
    \mathbb{P}\left( \exists t : \, Z_t^\epsilon \geq \log \frac{2}{\delta}, \, \theta_* < \hat{\theta}_t - \epsilon \right)
    \leq \mathbb{P}\left( \exists t : \, M_t^- \geq \frac{2}{\delta} \right)
    \leq \frac{\delta}{2}.
\end{align}

\paragraph{Step 3. Showing that $\lim_{\delta \to 0} \mathbb{P}_{\theta_*}^{\pi_{\epsilon, \delta}^\texttt{FIT-Q}} \left( | \hat{\theta}_\tau - \theta_* | > \eta \right) = 0$ for any $\eta > 0$.}
Condition~\ref{cond:FC-cons-cgf-dominate} implies that $\phi(\lambda_*, \cdot)$ is a non-negative function.
Therefore, for any $\theta$,
\begin{align}
    Z_t^\text{in}(\theta) 
    &= \sum_{s=1}^{t} \left[ \lambda_* \big| f(\hat{\theta}_t - X_s) - f(\theta - X_s) \big| - \phi(\lambda_*, f(\theta - X_s)) \right]
    \\&\leq \sum_{s=1}^{t}  \lambda_* \big| f(\hat{\theta}_t - X_s) - f(\theta - X_s) \big|
    \\&\leq \lambda_*t.
\end{align}
That is, $Z_t^\epsilon$ cannot grow faster than $\lambda_*$, and thus, $\tau \geq \log (2/\delta)/ \lambda_*$.

Note that \texttt{FIT-Q} query rule does not depend on $\epsilon$ or $\delta$. Therefore, we can consider a sequence of stopping times $\{ \tau_\delta \}_{\delta > 0}$, where $\tau_\delta$ denotes the stopping time given $\delta$, that are defined on the same probability space and share the same test statistic process $\{Z_t^\epsilon\}_{t \geq 0}$.
Above argument shows that when $\delta \to 0$, $\tau_\delta \to \infty$ almost surely.
Finally, by the strong consistency of MME (Proposition \ref{prop:MME-consistency}), we derive that $\hat{\theta}_{\tau_\delta} \to \theta_*$ almost surely as $\delta \to 0$, which concludes the proof.
\end{proof}

\newpage

\subsection{Proof of Theorem \ref{thm:FIT-Q-FC-optimality}}

\begin{proposition}\label{prop:maximal}
    Under any query rule, the following holds for any $t_0 \in \mathbb{N}$ and $\eta>0$:
	\begin{equation}
		\mathbb{P} \left( \exists t \geq t_0 : |\hat{\theta}_t - \theta_* | \geq \eta \right) \leq 2 \exp \big(- 2\kappa_\eta^2 t_0 \big),
	\end{equation}
    where $\hat{\theta}_t$ is the MME, and $\kappa_\eta \coloneqq \inf_{x \in \mathcal{X}} \min_{\theta \in \{\theta_* + \eta, \theta_* - \eta\}} | f(\theta - x) - f(\theta_* - x) |$.
\end{proposition}

\begin{proof}

We can decompose the probability as:
\begin{equation}
    \mathbb{P} \left( \exists t \geq t_0 : |\hat{\theta}_t - \theta_* | \geq \eta \right) = \underbrace{ \mathbb{P} \left( \exists t \geq t_0 : \hat{\theta}_t \geq \theta_* + \eta \right) }_{\text{(A)}} + \underbrace{ \mathbb{P} \left( \exists t \geq t_0 : \hat{\theta}_t \leq \theta_* - \eta \right) }_{\text{(B)}}.
\end{equation}
In what follows, we show that both (A) and (B) are bounded by $\exp \big(- 2\kappa_\eta^2 t_0 \big)$.

\paragraph{Step 1. Bounding the term (A).}
For any $\lambda > 0$, we have
\begin{align}
    &\left\{ \hat{\theta}_t \geq \theta_* +\eta \right\}
    \\&\stackrel{\text{(a)}}{=} \left\{ \sum_{s=1}^{t} f(\hat{\theta}_t - X_s) \geq \sum_{s=1}^{t} f(\theta_* - X_s + \eta) \right\}
    \\&\stackrel{\text{(b)}}{\subseteq} \left\{ \sum_{s=1}^{t} Y_s \geq \sum_{s=1}^{t} f(\theta_* - X_s + \eta) \right\}
    \\&= 	\left\{ \sum_{s=1}^{t} \Big[ \lambda \big(Y_s - f(\theta_* - X_s) \big) - \psi(\lambda; f(\theta_* - X_s)) \Big] \geq \sum_{s=1}^{t} \Big[ \lambda \big( f(\theta_* - X_s + \eta) - f(\theta_* - X_s) \big) - \psi(\lambda; f(\theta_* - X_s)) \Big] \right\}
    \\&\stackrel{\text{(c)}}{\subseteq} \left\{ \sum_{s=1}^{t} \Big[ \lambda \big(Y_s - f(\theta_* - X_s) \big) - \psi(\lambda; f(\theta_* - X_s)) \Big] \geq \sum_{s=1}^{t} \left(\lambda \kappa_\eta -\frac{1}{8}\lambda^2 \right)  \right\},
\end{align}
where step (a) uses the monotonicity of $f(\cdot)$, step (b) uses the fact that $\sum_{s=1}^{t} Y_s \geq \sum_{s=1}^{t} f(\hat{\theta}_t - X_s)$ when $\hat{\theta}_t \geq \theta_* +\eta$, and step (c) uses the definition of $\kappa_\eta$ together with Lemma \ref{lemma:cgf}.

Pick $\lambda = 4\kappa_\eta$ so that $\lambda \kappa_\eta -\frac{1}{8}\lambda^2 = 2\kappa_\eta^2$, and define
\begin{equation}
	M_t^+ \coloneqq \exp \left( \sum_{s=1}^{t} \Big[ \lambda \big( Y_s - f(\theta_* - X_s) \big) - \psi(\lambda; f(\theta_* - X_s)) \Big] \right).
\end{equation}
Then, the above result reduces to
\begin{equation}
    \left\{ \hat{\theta}_t \geq \theta_* +\eta \right\}
    \subseteq
    \left\{ \log M_t^+ \geq 2 \kappa_\eta^2 t \right\}
    = \left\{ M_t^+ \geq \exp(2 \kappa_\eta^2 t) \right\}.
\end{equation}
Since $(M_t^+)_{t \in \mathbb{N}}$ is a martingale with $M_0^+=1$, by Ville's inequality (Lemma~\ref{lemma:Ville}), we conclude that
\begin{equation}
    \mathbb{P} \left( \exists t \geq t_0 : \hat{\theta}_t \geq \theta_* + \eta \right)
    \leq \mathbb{P} \left( \exists t \geq t_0 : M_t^+ \geq \exp(2 \kappa_\eta^2 t) \right)
    \leq \exp\left( - 2\kappa_\eta^2 t_0 \right).
\end{equation}

\paragraph{Step 2. Bounding the term (B).}

Similarly to above, we can show that, for any $\lambda > 0$,
\begin{equation}
	\left\{ \hat{\theta}_t \leq \theta_* - \eta \right\} \subseteq \left\{ \sum_{s=1}^{t} \Big[ \lambda \big( f(\theta_* - X_s) - Y_s \big) - \psi(\lambda; f(\theta_* - X_s)) \Big] \geq \sum_{s=1}^{t} \left(\lambda \kappa_\eta -\frac{1}{8}\lambda^2\right)  \right\}.
\end{equation}
Picking $\lambda = 4\kappa_\eta$ and defining
\begin{equation}
    M_t^- \coloneqq \exp \left( \sum_{s=1}^{t} \Big[ \lambda \big( f(\theta_* - X_s) - Y_s \big) - \psi(-\lambda; f(\theta_* - X_s)) \Big] \right)
\end{equation}
gives
\begin{equation}
    \mathbb{P} \Big( \exists t \geq t_0 : \hat{\theta}_t \leq \theta_* - \eta \Big)
    \leq \mathbb{P} \left( \exists t \geq t_0 : M_t^- \geq \exp(2 \kappa_\eta^2 t) \right)
    \leq \exp\left( - 2\kappa_\eta^2 t_0 \right).
\end{equation}
\end{proof}

\newpage

\begin{proposition}
    Consider the setting of Theorem~\ref{thm:FIT-Q-FC-optimality}.
    Then, for any $\epsilon > 0$, there exists $\delta_0 > 0$ such that $\sup_{\delta \in (0, \delta_0)} \mathbb{E}_{\theta_*}^{\pi_{\epsilon,\delta}^\texttt{FIT-Q}} \left[ \frac{ \tau^2 }{\log(1/\delta)^2 } \right] < \infty$.
\end{proposition}
\begin{proof}
    Define a function $c:\mathbb{R} \to \mathbb{R}$,
    \begin{equation}
        c(\xi) := \min_{z \in \{z_*+\xi-\epsilon, z_*+\xi + \epsilon\}} \big[\lambda_* |f(z_* + \xi) - f(z)| - \phi(\lambda_*, f(z))\big].
    \end{equation}
    Lemma~\ref{lemma:ts-convergence} shows that $Z_t^\epsilon/t$ converges to $c_* = c(0)$ almost surely, and Condition~\ref{cond:FC-cons-positive} guarantees that $c(0) > 0$.
    Also note that $c(\cdot)$ is continuous due to Condition~\ref{cond:FC-cons-quasi-convex}.
    Therefore, there exists $\xi_0 > 0$ such that
    \begin{equation}
        \min_{\xi \in [-\xi_0, +\xi_0]} c(\xi) > 0.
    \end{equation}
    Note that the value of $\xi_0$ is irrelevant to $\delta$.

We define a good event as
\begin{equation}
	\mathcal{G}_t \coloneqq \left\{ \forall s \in \{ \lceil \sqrt{t} \rceil, \cdots, t\}, \; | \hat{\theta}_s - \theta_* | < \frac{\xi_0}{2} \right\}.
\end{equation}
Then, we decompose $\mathbb{E}[\tau_\delta^2]$ as:
\begin{align}
	\mathbb{E} [\tau_\delta^2] =\; & \sum_{t=0}^{\infty} (2t + 1) \mathbb{P} (\tau_\delta > t)\\
	=\; &  \sum_{t=0}^{\infty} (2t+1) \mathbb{P} \left( \tau_\delta > t , \;\mathcal{G}_t^c \right) + \sum_{t=0}^{\infty} (2t + 1) \mathbb{P} \left( \tau_\delta > t , \; \mathcal{G}_t \right)\\
	\leq\; & \sum_{t=0}^{\infty} (2t+1) \mathbb{P} \left( \mathcal{G}_t^c \right) + \sum_{t=0}^{\infty} (2t + 1) \mathbb{P} \left( \tau_\delta > t , \; \mathcal{G}_t \right).
\end{align}
To obtain the desired result, it is sufficient to show that
\begin{align}
    \underbrace{ \limsup_{\delta \to 0} \frac{ \sum_{t=0}^{\infty} (2t+1) \mathbb{P} \left( \mathcal{G}_t^c \right) }{ \log(1/\delta) } }_{\text{(A)}} < \infty
    , \quad
    \underbrace{ \limsup_{\delta \to 0} \frac{ \sum_{t=0}^{\infty} (2t + 1) \mathbb{P} \left( \tau_\delta > t , \; \mathcal{G}_t \right) }{ \log(1/\delta) } }_{\text{(B)}} < \infty
\end{align}

\paragraph{Step 1. Showing that the term (A) is finite.}
Proposition \ref{prop:maximal} implies that 
\begin{equation}
    \mathbb{P}(\mathcal{G}_t^c)
        \leq \mathbb{P} \left( \exists s \geq \lceil \sqrt{t} \rceil : | \hat{\theta}_s - \theta_* | \geq \frac{\xi_0}{2} \right)
        \leq
        2 \exp \left( - 2 \kappa_{\xi_0/2}^2 \sqrt{t} \right),
\end{equation}
for some constant $\kappa_{\xi_0/2} > 0$.
Therefore,
\begin{equation}
    \sum_{t=0}^{\infty} (2t + 1) \mathbb{P} (\mathcal{G}_t^c)
    \leq
    \sum_{t=0}^{\infty} (2t + 1) \cdot 2 \exp \left( - 2 \kappa_{\xi_0/2}^2 \sqrt{t} \right).
\end{equation}
Note that the right hand side is a finite term, which does not have any dependence on $\delta$.
Therefore, the term (A) is $0$.

\paragraph{Step 2. Showing that the term (B) is finite.}
We begin by establishing a deterministic lower bound on $Z_t^\epsilon$ under event $\mathcal{G}_t$.
Observe that
\begin{align}
    Z_t^{\epsilon} &\coloneqq \min_{\theta \in \{\hat{\theta}_t-\epsilon, \hat{\theta}_t+\epsilon\}} \sum_{s=1}^t \left[ \lambda_* \left| f(\hat{\theta}_t - X_s) - f(\theta - X_s) \right| - \phi( \lambda_*, f(\theta - X_s) ) \right]
    \\&\geq \sum_{s=1}^t \min_{\theta \in \{\hat{\theta}_t-\epsilon, \hat{\theta}_t+\epsilon\}} \left[ \lambda_* \left| f(\hat{\theta}_t - X_s) - f(\theta - X_s) \right| - \phi( \lambda_*, f(\theta - X_s) ) \right]
    \\&= \sum_{s=1}^t c(\hat{\theta}_t-X_s - z_*).
\end{align}
Also note that
\begin{equation}
    \inf_{\xi \in \mathbb{R}} c(\xi) \geq \inf_{z \in \mathbb{R}}\left\{ -\phi(\lambda_*, f(z)) \right\}
    \geq - \max_{p \in [0,1]} \phi(\lambda_*, p).
\end{equation}

On event $\mathcal{G}_t$, \texttt{FIT-Q} algorithm must have chosen the $\frac{\xi_0}{2}$-optimal queries after the moment $\lceil \sqrt{t} \rceil$, and therefore,
\begin{equation}
    |\hat{\theta}_t-X_s - z_*|
    = |(\hat{\theta}_t-X_s) - (\theta_* - x_*)|
    \leq |\hat{\theta}_t - \theta_*| + |X_s - x_*|
    \leq \frac{\xi_0}{2} + \frac{\xi_0}{2}
    = \xi_0,
\end{equation}
for $s=\lceil \sqrt{t} \rceil+1, \ldots, t$.

Define $c_0 := \min_{\xi \in [-\xi_0, +\xi_0]} c(\xi) > 0$ and $\phi_0 := \max_{p \in [0,1]} \phi(\lambda_*, p) < \infty$.
Combining above results, we obtain that, on event $\mathcal{G}_t$,
\begin{align}
    Z_t^{\epsilon} &= \sum_{s=1}^t c(\hat{\theta}_t-X_s - z_*)
    \\&= \sum_{s=1}^{\lceil \sqrt{t} \rceil} c(\hat{\theta}_t-X_s - z_*)
        + \sum_{s=\lceil \sqrt{t} \rceil+1}^{t} c(\hat{\theta}_t-X_s - z_*)
    \\&\geq \sum_{s=1}^{\lceil \sqrt{t} \rceil} -\phi_0
        + \sum_{s=\lceil \sqrt{t} \rceil+1}^{t} c_0
    \\&\geq - (\sqrt{t}+1) \phi_0 + (t-\sqrt{t}-1) c_0.
\end{align}
Let $t_0 := \max\left\{ 10, 64\phi_0^2/c_0^2, \frac{\log(2/\delta)}{c_0/4} \right\}$.
Then, for any $t \geq t_0$,
\begin{align}
    - (\sqrt{t}+1) \phi_0 + (t-\sqrt{t}-1) c_0
    \quad \stackrel{\text{(a)}}{\geq} \quad
    -2\sqrt{t} \cdot \phi_0 + t/2 \cdot c_0
    \quad \stackrel{\text{(b)}}{\geq} \quad t/4 \cdot c_0
    \quad \stackrel{\text{(c)}}{\geq} \quad
    \log \frac{2}{\delta},
\end{align}
where step (a) uses that $\sqrt{t}+1 \leq 2\sqrt{t}$ and $t-\sqrt{t}-1 \geq t/2$ for any $t \geq 10$, step (b) uses that $t/4 \cdot c_0 \geq 2 \sqrt{t} \cdot \phi_0$ for any $t \geq 64\phi_0^2/c_0^2$, and step (c) uses that $t/4 \cdot c_0 \geq \log(2/\delta)$ for any $t \geq \frac{\log(2/\delta)}{c_0/4}$.
And hence
\begin{equation}
    \mathbb{P} \left( Z_t^\epsilon < \log \frac{2}{\delta}, \; \mathcal{G}_t \right) 
    \leq \mathbb{P} \left( - (\sqrt{t}+1) \phi_0 + (t-\sqrt{t}-1) c_0 < \log \frac{2}{\delta}, \; \mathcal{G}_t \right) 
    \leq \mathbb{I}\{ t \leq t_0 \}.
\end{equation}

Finally, we can deduce that
\begin{align}
    \sum_{t=0}^{\infty} (2t + 1) \mathbb{P} \left( \tau > t , \; \mathcal{G}_t \right) 
    \leq\; & \sum_{t=0}^{\infty} (2t + 1) \mathbb{P} \left( Z_t^\epsilon < \log \frac{2}{\delta}, \; \mathcal{G}_t \right)\\ 
    \leq\; &  \sum_{t=0}^{\infty} (2t + 1) \mathbb{I}\{ t \leq t_0 \}\\
    =\; & \sum_{t=0}^{\lfloor t_0 \rfloor} (2t+1)\\
    =\; & (\lfloor t_0 \rfloor +1)^2,
\end{align}
where the first inequality follows from that $\{ \tau > t \} \subseteq \{ Z_t^\epsilon < \log \frac{2}{\delta} \}$.
We conclude that the term (B) is finite by observing that
\begin{equation}
    \limsup_{\delta \to 0} \frac{t_0}{\log(1/\delta)} 
    = \limsup_{\delta \to 0} \frac{\max\left\{ 10, 64\phi_0^2/c_0^2, \frac{\log(2/\delta)}{c_0/4} \right\}}{\log(1/\delta)} 
    = \frac{c_0}{4} < \infty.
\end{equation}

\end{proof}

\begin{proof}[\textbf{Proof of Theorem \ref{thm:FIT-Q-FC-optimality}}]
As done in the proof of Theorem \ref{thm:FIT-Q-FC-consistency}, we consider the sequence of stopping times $\{ \tau_\delta \}_{\delta > 0}$ that are defined on the same probability space and share the same test statistics process $\{ Z_t^\epsilon \}_{t \in \mathbb{N}}$.
This is possible since \texttt{FIT-Q} query rule does not depend on $\delta$.

Define
\begin{equation}
    c_* := \min_{z \in \{z_*-\epsilon, z_* + \epsilon\}} \big[\lambda |f(z) - f(z_*)| - \phi(\lambda, f(z))\big].
\end{equation}
Lemma~\ref{lemma:ts-convergence} shows that $Z_t^\epsilon/t \to c_*$ almost surely as $t \to \infty$.
Fix a random outcome $\omega$ at which $Z_t^\epsilon(\omega)/t$ converges to $c_*$.
For any $\eta \in (0, c_*)$, there exists $T_\eta(\omega) \in \mathbb{N}$ such that $Z_t^\epsilon(\omega)/t \geq c_* - \eta$ for all $t \geq T_\eta(\omega)$, and thus $\tau_\delta(\omega) \leq \max\{ \frac{ \log(2/\delta) }{c_* - \eta }, T_\eta(\omega) \}$.
Since $T_\eta(\omega)$ does not depend on $\delta$, we deduce that
\begin{equation}
    \limsup_{\delta \to 0} \frac{\tau_\delta(\omega)}{\log(1/\delta)}
        \leq \lim_{\delta \to 0} \frac{ \log(2/\delta)/(c_*-\eta) }{ \log(1/\delta) }
        = \frac{1}{c_*-\eta}.
\end{equation}
Since the choice of $\eta$ was arbitrary,
\begin{equation}
    \limsup_{\delta \to 0} \frac{\tau_\delta}{ \log(1/\delta)} \leq \frac{1}{c_*}
    \quad \text{almost surely}.
\end{equation}

By Proposition~\ref{prop:maximal}, the sequence $\{\frac{\tau_\delta}{\log \frac{1}{\delta}} \}_\delta$ is uniformly integrable.
Therefore,
\begin{equation}
    \limsup_{\delta \to 0}\frac{\mathbb{E}_{\theta_*}^{\pi_{\epsilon,\delta}^\texttt{FIT-Q}} \left[ \tau_\delta \right]}{ \log (1/\delta)} 
    \leq \frac{1}{c_*}.
\end{equation}

Now, we show that $c_* = I(x_*; \theta_*) \epsilon^2/2 + o(\epsilon^2)$ by showing that for $z \in \{ z_* - \epsilon, z_* +\epsilon \}$
\begin{equation}
    \underbrace{\lim_{\epsilon\to 0} \frac{\lambda_*(\epsilon) | f(z) - f(z_*) |}{\epsilon^2} }_{(A)}
    = \frac{g_*^2}{v_*}
    , \quad
    \underbrace{ \lim_{\epsilon\to 0} \frac{\phi(\lambda_*(\epsilon), f(z))}{\epsilon^2} }_{(B)}
    = \frac{g_*^2}{2v_*},
\end{equation}
where $g_* = f'(z_*)$, and $v_* = f(z_*) (1-f(z_*))$.
Fix $z = z_* + \epsilon$.
For the term (A), we have
\begin{equation}
    \lim_{\epsilon\to 0} \frac{\lambda_*(\epsilon) | f(z_*+\epsilon) - f(z_*) |}{\epsilon^2} 
    = \lim_{\epsilon\to 0} \frac{\lambda_*(\epsilon)}{\epsilon} \cdot \frac{| f(z_*+\epsilon) - f(z_*) |}{\epsilon}
    = \frac{g_*^2}{v_*},
\end{equation}
where the last equality follows from Condition~\ref{cond:FC-opt-lambda}.
Condition~\ref{cond:FC-opt-phi} implies that there exists a constant $C$ such that 
\begin{equation}
    \left| \phi(\lambda_*(\epsilon), f(z_*+\epsilon)) - \frac{\lambda_*^2(\epsilon)}{2} f(z_*+\epsilon)(1-f(z_*+\epsilon)) \right| 
        \leq C\lambda_*^3(\epsilon),
\end{equation}
for any $\epsilon > 0$.
Note that
\begin{equation}
    \lim_{\epsilon \to 0} \frac{ \frac{\lambda_*^2(\epsilon)}{2} f(z_*+\epsilon)(1-f(z_*+\epsilon)) }{\epsilon^2}
    = \frac{g_*^2}{2v_*},
\end{equation}
and
\begin{equation}
    \lim_{\epsilon \to 0} \frac{\lambda_*^3(\epsilon)}{\epsilon^2} = \lim_{\epsilon \to 0} \epsilon \cdot \left( \frac{\lambda_*(\epsilon)}{\epsilon} \right)^3 = 0,
\end{equation}
by Condition~\ref{cond:FC-opt-lambda}.
Combining these results, we obtain
\begin{equation}
    \lim_{\epsilon\to 0} \frac{\phi(\lambda_*(\epsilon), f(z_*+\epsilon))}{\epsilon^2} = \frac{g_*^2}{2v_*},
\end{equation}
which completes the proof for the term (B).
We can derive the same conclusion for $z=z_*-\epsilon$.
\end{proof}

\subsubsection{Proof of Proposition \ref{prop:verification}}
\begin{proposition}
    Consider $\phi$ and $\lambda_*$ defined in Section~\ref{section:algorithm}:
\begin{equation}
    \phi(\lambda, p) \coloneqq \big( e^\lambda - \lambda - 1 \big) p (1-p),
\end{equation}
and
\begin{equation}
    \lambda_* \coloneqq \argmax_{\lambda \in (0,\overline{\lambda}]} \min_{z \in \{z_*-\epsilon, z_*+\epsilon\} } \left[ \lambda \big| f(z) - f(z_*) \big| - \phi(\lambda, f(z)) \right],
\end{equation}
with $\overline{\lambda} \coloneqq \sup_{\lambda > 0} \{ \lambda \geq e^\lambda - \lambda -1 \} \approx 1.2564$.
    Then, this choice of $\phi$ and $\lambda$ satisfies all conditions in Theorems~\ref{thm:FIT-Q-FC-consistency} and \ref{thm:FIT-Q-FC-optimality}.
\end{proposition}

\begin{proof}
We aim to verify Conditions \ref{cond:FC-cons-quasi-convex}--\ref{cond:FC-cons-positive} of Theorem~\ref{thm:FIT-Q-FC-consistency} and Conditions~\ref{cond:FC-opt-phi}--\ref{cond:FC-opt-lambda} of Theorem~\ref{thm:FIT-Q-FC-optimality}.

\paragraph{Verification of Condition \ref{cond:FC-cons-quasi-convex}.}
Define
\begin{align}
    h(p) :=\; & \lambda_* | p - p_0 | - \phi(\lambda_*, p)\\
    =\; &
    \begin{cases}
        \lambda_* ( p - p_0 )- \left(e^{\lambda_*} - \lambda_* - 1\right) p (1-p),& \text{if } p \geq p_0,\\
        -\lambda_* ( p - p_0 )- \left(e^{\lambda_*} - \lambda_* - 1\right) p (1-p), & \text{if } p< p_0.
    \end{cases}
\end{align}
Note that $|p-p_0|$ and $\phi(\lambda_*,p)$ are continuous and convex in $p$. Therefore, $h$ is also continuous and convex in $p$.
Also observe that
\begin{equation}
    h'(p) =
    \begin{cases}
        \lambda_* - \left(e^{\lambda_*} - \lambda_* - 1 \right) (1 - 2p), & \text{if } p > p_0, \\
        -\lambda_* - \left(e^{\lambda_*} - \lambda_* - 1 \right) (1 - 2p), & \text{if } p < p_0.
    \end{cases}
\end{equation}
Since $|1-2p| \leq 1$ and $|e^{\lambda} - \lambda - 1| \leq \lambda$ for any $\lambda \leq \bar{\lambda}$, we have $h'(p) \geq 0$ if $p > p_0$ and $h'(p) \leq 0$ if $p < p_0$.
Therefore, $h(p)$ is minimized at $p=p_0$.

\paragraph{Verification of Condition \ref{cond:FC-cons-cgf-dominate}.}
It is shown in Lemma \ref{lemma:cgf}.

\paragraph{Verification of Condition \ref{cond:FC-cons-positive}.}
We first consider a function $\rho(\lambda) = a \lambda  - b \left( e^\lambda - \lambda - 1\right)$, where $a, b >0$, then $\rho(0) = 0$, $\lim_{\lambda \to \infty} \rho(\lambda) = -\infty$, $\rho'(0)= a > 0$ and $\rho''(\lambda) = -b e^\lambda <0$.
Therefore, it is concave and it has positive value when $\lambda$ is between 0 and some positive
threshold (see the dashed curves in Figure \ref{fig:lambda}).

Now, we consider two functions $\lambda \mapsto \lambda | f(z_*+\epsilon) - f(z_*)| - \phi(\lambda, f(z_* + \epsilon))$ and $\lambda \mapsto \lambda | f(z_*-\epsilon) - f(z_*)| - \phi(\lambda, f(z_* - \epsilon))$, then by the above argument, there is an interval $(0, \lambda_0) \subseteq (0, \overline{\lambda})$ on which two functions are positive at the same time.
Therefore, the maximum value of $\min_{z \in \{ z_* +\epsilon, z_* -  \epsilon\}} \left[ \lambda \left| f(z) - f(z_*) \right| - \phi(\lambda, f(z)) \right]$ is positive as
\begin{align}
    \max_{\lambda \in(0, \overline{\lambda}]}\min_{z \in \{ z_* +\epsilon, z_* -  \epsilon\}} \left[ \lambda \left| f(z) - f(z_*) \right| - \phi(\lambda, f(z)) \right] \geq\; & \max_{\lambda \in(0, \lambda_0)}\min_{z \in \{ z_* +\epsilon, z_* -  \epsilon\}} \left[ \lambda \left| f(z) - f(z_*) \right| - \phi(\lambda, f(z)) \right]\\
    >\; & \max_{\lambda \in (0 ,\lambda_0)} 0 = 0.
\end{align}

\paragraph{Verification of Condition \ref{cond:FC-opt-phi}.}
Note that 
\begin{align}
    \sup_{p \in [0,1]} \left| \phi(\lambda, p) - \frac{\lambda^2}{2} p(1-p) \right| =\; & \sup_{p \in [0,1]} \left| \left(e^\lambda - \lambda - 1 \right) p(1-p) - \frac{\lambda^2}{2} p(1-p) \right|\\
    =\; & \frac{1}{4} \left| \left(e^\lambda - \lambda - 1 \right) - \frac{\lambda^2}{2} \right|,
\end{align}
and this value is $o(\lambda^2)$ by Taylor's theorem.

\paragraph{Verification of Condition \ref{cond:FC-opt-lambda}.}
Define two thresholds $\lambda_+$ and $\lambda_-$ as 
\begin{equation}
    \lambda_+(\epsilon) \coloneqq \argmax_{\lambda \in \mathbb{R}^+} \left[ \lambda \big| f(z_* + \epsilon) - f(z_*) \big| - \phi(\lambda, f(z_* + \epsilon)) \right]
\end{equation}
and
\begin{equation}
    \lambda_-(\epsilon) \coloneqq \argmax_{\lambda \in \mathbb{R}^+} \left[ \lambda \big| f(z_* - \epsilon) - f(z_*) \big| - \phi(\lambda, f(z_* - \epsilon)) \right],
\end{equation}
then they can be calculated in closed form as
\begin{equation}
    \lambda_+(\epsilon) = \log \left( 1+ \frac{|f(z_*+\epsilon) - f(z_*)|}{f(z_*+\epsilon)\big(1 - f(z_* +\epsilon)\big)} \right)\quad\text{and}\quad
    \lambda_-(\epsilon) = \log \left( 1+ \frac{|f(z_*-\epsilon) - f(z_*)|}{f(z_*-\epsilon)\big(1 - f(z_* -\epsilon)\big)} \right).
\end{equation}
Note that $\lambda_+(\epsilon)$ and $\lambda_-(\epsilon)$ are approximated by $\epsilon \times f'(z_*) / (f(z_*)(1-f(z_*)))$ with error terms of order $o(\epsilon)$ since we have
\begin{align}
    \lim_{\epsilon \to 0} \frac{\lambda_+(\epsilon)}{\epsilon} =\; & \lim_{\epsilon \to 0} \frac{\log \left( 1+ \frac{|f(z_*+\epsilon) - f(z_*)|}{f(z_*+\epsilon)\big(1 - f(z_* +\epsilon)\big)} \right)}{\epsilon}\\
    =\; & \lim_{\epsilon \to 0}\frac{\log \left( 1+ \frac{|f(z_*+\epsilon) - f(z_*)|}{f(z_*+\epsilon)\big(1 - f(z_* +\epsilon)\big)} \right)}{\frac{|f(z_*+\epsilon) - f(z_*)|}{f(z_*+\epsilon)\big(1 - f(z_* +\epsilon)\big)}}\frac{\frac{|f(z_*+\epsilon) - f(z_*)|}{f(z_*+\epsilon)\big(1 - f(z_* +\epsilon)\big)}}{\epsilon}\\
    =\; & \frac{f'(z_*)}{f(z_*) \big( 1 - f(z_*) \big)}.
\end{align}
Consider two functions $h_+(\lambda) = \lambda | f(z_*+\epsilon) - f(z_*) | - \phi(\lambda, f(z_*+\epsilon))$ and $h_-(\lambda) = \lambda | f(z_*+\epsilon) - f(z_*) | - \phi(\lambda, f(z_*+\epsilon))$, then they are concave, and their maximizers are $\lambda_+(\epsilon)$ and $\lambda_-(\epsilon)$, respectively.
Since two functions are concave, their minimum
\begin{equation}
    h(\lambda) = \min_{z \in \{z_*+\epsilon, z_* - \epsilon\}} \lambda | f(z) - f(z_*) | - \phi(\lambda, f(z))
\end{equation}
is also concave.
Moreover, $h$ increases in $\lambda < \min \{ \lambda_+(\epsilon), \lambda_-(\epsilon) \}$ since both of $h_+$ and $h_-$ are decreasing there; and similarly, $h$ decreases in $\lambda > \max\{ \lambda_+(\epsilon), \lambda_-(\epsilon) \}$.
These imply that the maximizer of $h$, $\lambda_*$, exists and $\min\{\lambda_+(\epsilon), \lambda_-(\epsilon)\} \leq \lambda_* \leq \max\{\lambda_+(\epsilon), \lambda_-(\epsilon)\}$.
Finally, we can conclude that the second condition is satisfied when $\epsilon$ is sufficiently small.

\end{proof}

\newpage
\section{Comparison with\cite{bassamboo2023learning}}\label{appendix:bassamboo}

Our work shares a common foundation with the research of \cite{bassamboo2023learning}, as both studies address the adaptive testing problem by drawing insights from the multi-armed bandit framework.
The similarities are significant: both aim to estimate an unknown continuous ability parameter by asking a sequence of questions, where the choice of the next question (or ``arm'') is dynamically adapted based on the candidate's previous responses.
Furthermore, both models rely on a probabilistic response function to link the candidate's ability and the question's difficulty to the binary outcome.
The overarching goal in both papers is to achieve this estimation with the fewest questions possible while satisfying a predefined statistical guarantee, often referred to as a $\delta$-correct framework.

\paragraph{Difference in objectives \& test statistic design.}
Despite these parallels, a fundamental difference in the objective distinguishes our work.
\cite{bassamboo2023learning} focus on interval ``classification'': their goal is to determine which of several predefined and fixed intervals (or grades) contains the candidate's ability.
In contrast, our paper tackles the problem of continuous ``point estimation'': we aim to identify the specific value of the ability parameter within a given error margin $\epsilon$, such that the target interval $[\theta_* - \epsilon, \theta_* + \epsilon]$ is centered on the unknown true parameter itself, not relying on a fixed partitioning.

This fundamental difference in objectives naturally leads to distinct methodological approaches, particularly in the design of their respective test statistics.
The test statistic in \cite{bassamboo2023learning} is designed to confirm that the true parameter lies within a fixed interval.
This is achieved by computing a generalized likelihood ratio that compares the empirical mean of responses for each question difficulty (or ``arm") to the predicted mean under an alternative hypothesis outside the target interval.
In contrast, our test statistic is designed to certify that the true parameter lies within a dynamic, $\epsilon$-neighborhood of our current estimate.
It evaluates the cumulative evidence against alternatives outside this neighborhood by leveraging the global structure of the response model, comparing predicted outcomes under the current estimate against those under an alternative, rather than relying on per-arm empirical means.
This design allows our statistic to efficiently exploit the relationship between questions and, due to its quasi-convex property, simplifies the computationally demanding search for the most challenging alternative to a simple comparison at the boundary points $\hat{\theta}_t \pm \epsilon$.

\paragraph{Adapting \cite{bassamboo2023learning} via a discretization trick: experiment \& analysis.}
To investigate the implications of these methodological differences, we adapt their suggested algorithm, \texttt{A1}, to our continuous point-estimation setting using a discretization trick.
This involves partitioning the continuous parameter space into a series of disjoint subintervals, each of width $2 \epsilon$.
By treating these subintervals as the ``grades'' to be classified, we transform our problem into an interval classification task that \texttt{A1} is designed to solve.
When the algorithm terminates and identifies a target interval, we take its midpoint as the final estimate $\hat{\theta}_\tau$.
This construction ensures that the $(\epsilon, \delta)$-correctness guarantee of our problem is met.
However, while this adaptation makes \texttt{A1} applicable, the following analysis reveals that it suffers from critical performance drawbacks, most notably a severe sensitivity to the true parameter's location relative to the discretization boundaries.

To demonstrate this sensitivity empirically, we conduct a numerical experiment under the logistic response model.
We set the error margin to $\epsilon=0.4$ and the confidence level to $\delta=0.05$.
The parameter space $\Theta = [-2,2]$ is partitioned into five subintervals of width $0.8$.
Accordingly, the discrete query set for \texttt{A1} is defined as the midpoints of these intervals: $\{-1.6,-0.8,0,0.8,1.6\}$. 
We then vary the true parameter, $\theta_*$, within one of these subintervals, selecting values from $\{-0.3, -0.2, -0.1, 0, 0.1, 0.2, 0.3\}$.
For each case, we compare the mean stopping time ($\mathbb{E}[\tau]$) of the adapted \texttt{A1} algorithm against that of \texttt{FIT-Q}.

\begin{figure}[h!]
    \centering
    \includegraphics[width=0.6\linewidth]{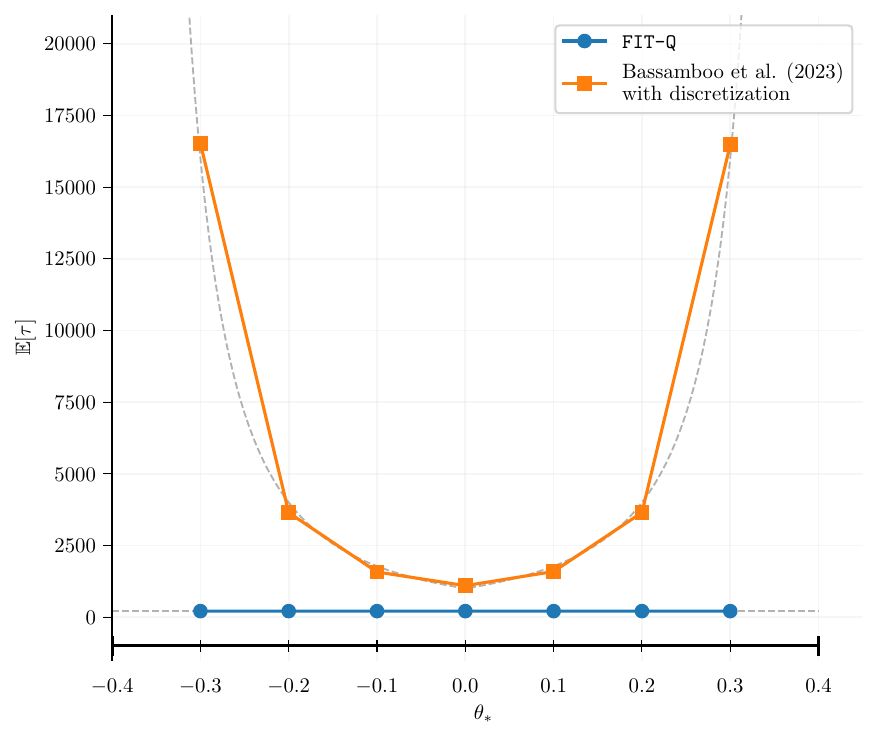}
    \caption{Comparison of average stopping times for \texttt{FIT-Q}, and the \texttt{A1} algorithm suggested in \cite{bassamboo2023learning} adapted to our setting via discretization.
    The plot measures the mean stopping time as a function of the true parameter $\theta_*$.
    The x-axis covers the range $[-0.4, 0.4)$, which is one of the subintervals created by partitioning the full parameter space $\Theta = [-2,2]$.}
    \label{fig:compare}
\end{figure}

The results, visualized in Figure~\ref{fig:compare}, reveal a stark difference in performance.
The \texttt{FIT-Q} algorithm exhibits remarkable efficiency, recording a consistently low mean stopping time of approximately $210$, regardless of the true parameter's value.
In contrast, the adapted \texttt{A1} algorithm is significantly less efficient.
Even at its best when $\theta_*$ is at the center of the interval, it requires an average of 1,100 samples, over five times more than \texttt{FIT-Q}.
This inefficiency becomes extreme as $\theta_*$ approaches the interval boundaries; at $\theta_* = \pm 0.3$, \texttt{A1}'s mean stopping time balloons to approximately 16,500 samples.
These findings empirically confirm that the discretization-based approach is highly sensitive to the parameter's location, a weakness not present in our natively continuous \texttt{FIT-Q} algorithm.

To understand the root cause of \texttt{A1}'s inefficiency, we now analytically examine its core components.
The performance degradation can be attributed to two primary factors: (1) the slow growth of its test statistic near interval boundaries, and (2) an overly conservative stopping threshold.

First, let us consider the test statistic suggested by \cite{bassamboo2023learning}, which for a discretized query set $\{x_1, \ldots, x_m\}$ is defined as\footnote{The expression is modified to be consistent with our notation.}
\begin{equation}
    Z_t^\texttt{A1} := \inf_{\theta \not\in \hat{I}_t}\sum_{k=1}^{m} \left\{ N_{k,t} \cdot d\left( \left. \frac{S_{k,t}}{N_{k,t}} \right| f(\theta - x_k) \right)\right\},
\end{equation}
where $\hat{I}_t$ is the interval containing the current estimate $\hat{\theta}_t$, $N_{k,t}$ is the number of times query $x_k$ has been selected up to time $t$, and $S_{k,t}$ is the number of positive responses for that query.
The purpose of this statistic is to accumulate evidence until the true interval $I_*$ can be confidently identified.
Consequently, its asymptotic growth rate that determines the overall sample complexity is governed by the KL divergence to the most challenging alternative hypothesis, which is the nearest boundary of $I_*$ itself.
If the query set is discritized sufficiently dense to contain the optimal query $x_*$, the limiting value can be approximated by $\min_{\theta \in \{\inf(I_*), \sup(I_*)\}} d( f(\theta_*-x_*) | f(\theta-x_*) )$.
Let us denote the relative distance of the true parameter $\theta_*$ from this nearest boundary as $\alpha \epsilon$, where $\alpha =1$ at the center and $\alpha \to 0$ as it approaches a boundary.
For sufficiently small $\epsilon$, this rate can be approximated by $\frac{I(x_*; \theta_*)}{2} \cdot (\alpha \epsilon)^2$.
Consequently, the required sample complexity to reach a fixed threshold inflates by a factor of $\alpha^{-2}$ as the true parameter nears a boundary, explaining the dramatic U-shaped curve in Figure~\ref{fig:compare}.

Second, this boundary-sensitivity issue is compounded by a conservative stopping rule. The algorithm terminates at
\begin{equation}
    \tau^\texttt{A1} := \inf\left\{ t \in \mathbb{N} \left| Z_t^\texttt{A1} \geq 3 \sum_{k=1}^{m} \log ( 1 + \log N_{k,t}) + m \mathcal{T}(\log(2/\delta)/m) \right. \right\},
\end{equation}
where $\mathcal{T}(x) = x + o(x)$.
The critical component here is the correction term, $3 \sum_{k=1}^{m} \log ( 1 + \log N_{k,t})$.
This term is necessary because the statistic $Z_t^\texttt{A1}$ relies on per-arm empirical means (i.e., $S_{k,t}/N_{k,t}$), 
considering each arm to be independent despite their interdependency mediated by the response model $f$.
While asymptotically negligible, this term is substantial in finite-sample settings. Its magnitude scales with $m$, the number of arms.
Since a small error margin $\epsilon$ requires a fine-grained discretization ($m=\Theta(\epsilon^{-1})$), the correction term forces the threshold to be excessively high.

In contrast, \texttt{FIT-Q}'s test statistic is built upon a single, global point-estimate $\hat{\theta}_t$.
This design choice obviates the need for a per-arm correction term, allowing for a much tighter stopping threshold of simply $\log(2/\delta)$.
This fundamental difference explains why \texttt{A1} is significantly less efficient than \texttt{FIT-Q}, even under the ideal condition of $\alpha=1$.

\paragraph{Computational efficiency.}
Beyond its statistical inefficiency, the adapted \texttt{A1} algorithm also poses computational challenges not present in \texttt{FIT-Q}.
At each step, \texttt{A1}'s query selection rule requires solving a max-min optimization problem to determine the next question to ask.
While this can be simplified under strong regularity conditions, it remains a considerable burden compared to \texttt{FIT-Q}'s query rule, which uses a simple and direct projection given in \eqref{eq:projection}.
Furthermore, a similar issue arises in the stopping rule. Calculating \texttt{A1}'s test statistic, $Z_t^\texttt{A1}$, requires finding an infimum over a continuous space of all alternative parameters outside the current target interval, a computationally demanding task.
In contrast, our test statistic was designed with a quasi-convex property.
This principled design allows the search for the most challenging alternative to be reduced to a simple evaluation at the two boundary points, $\hat{\theta}_t \pm \epsilon$.
These differences in both query selection and stopping condition make \texttt{FIT-Q} not only more sample-efficient but also substantially more computationally tractable and practical for real-time implementation.

\end{document}